\documentclass{article}
 \usepackage{arxiv} 

\usepackage{amsmath, amsfonts, amssymb}
\usepackage{amsthm} 
\usepackage{mathtools}
\usepackage{latexsym}
\usepackage{relsize}
\usepackage{enumitem}
\usepackage{xcolor}
\usepackage{dsfont}
\usepackage{times}
\usepackage{algorithm, algorithmic}
\usepackage{fullpage}
\usepackage{natbib}
\usepackage{tikz}
\usepackage{hyperref} 

\def\mA{{\mathcal A}}

\newcommand{\din}{d_{\text{in}}}
\newcommand{\dout}{d_{\text{out}}}

\newcommand{\defeq}{\stackrel{\text{def}}{=}}

\newcommand{\naman}[1]{}

\def\R{{\mathcal R}}

\newcommand{\A}{\mathcal{A}}

\newcommand{\K}{\ensuremath{\mathcal K}}

\def\regret{\mbox{{Regret}}}

\theoremstyle{plain}
\newtheorem{theorem}{Theorem}
\newtheorem{lemma}[theorem]{Lemma}

\newtheorem*{theorem*}{Theorem}
\newtheorem*{lemma*}{Lemma}
\newtheorem*{corollary*}{Corollary}
\newtheorem*{proposition*}{Proposition}
\newtheorem*{claim*}{Claim}
\newtheorem*{fact*}{Fact}
\newtheorem*{observation*}{Observation}
\newtheorem*{assumption*}{Assumption}

\theoremstyle{definition}
\newtheorem{definition}[theorem]{Definition}

\newtheorem*{definition*}{Definition}
\newtheorem*{remark*}{Remark}
\newtheorem*{example*}{Example}

\theoremstyle{plain}
\newtheorem*{theoremaux}{\theoremauxref}
\gdef\theoremauxref{1}

\DeclareMathAlphabet{\mathbfsf}{\encodingdefault}{\sfdefault}{bx}{n}

\DeclareMathOperator*{\argmin}{arg\,min}

\def\mA{{\mathcal A}}

\newcommand{\abs}[1]{|#1|}
\newcommand{\norm}[1]{\|#1\|}

\newcommand{\reals}{\mathbb{R}}

\renewcommand{\leq}{~\le~}
\renewcommand{\geq}{~\ge~}

\let\oldtfrac\tfrac
\renewcommand{\tfrac}[2]{\smash{\oldtfrac{#1}{#2}}}

\let\nablaold\nabla
\renewcommand{\nabla}{\nablaold\mkern-2.5mu}

\newcommand{\id}{I}
\newcommand{\mtrue}{M^{\textrm{true}}}

\newcommand{\mtrueconst}{r}

\newcommand{\proj}{\textrm{Proj}}
\newcommand{\unfairregrettext}{Asymmetric-Regret}
\newcommand{\unfairregret}{\mathrm{Regret}_{\mathrm{Asymmetric},T}}

\title{Provable Length Generalization in Sequence Prediction via Spectral Filtering}

\begin{document}

\author{
    Annie Marsden  \thanks{Equal contribution} \\ 
    \And
    Evan Dogariu  \thanks{Equal contribution}\\ 
    \And
    Naman Agarwal \\
    \And
    Xinyi Chen \\
    \AND
    Daniel Suo  \\
    \And
    Elad Hazan \thanks{Google DeepMind, \texttt{\{anniemarsden,dogariu,xinyic,namanagarwal,dsuo,ehazan\}@google.com} }\\
}

\maketitle

\begin{abstract}
We consider the problem of length generalization in sequence prediction. We define a new metric of performance in this setting -- the \unfairregrettext -- which measures regret against a benchmark predictor with longer context length than available to the learner. We continue by studying this concept through the lens of the spectral filtering algorithm. We present a gradient-based learning algorithm that provably achieves length generalization for linear dynamical systems. We conclude with proof-of-concept experiments which are consistent with our theory.
\end{abstract}

\section{Introduction}

Sequence prediction is a fundamental problem in machine learning with widespread applications in natural language processing, time-series forecasting, and control systems. In this setting, a learner observes a sequence of tokens and iteratively predicts the next token, suffering a loss that measures the discrepancy between the predicted and the true token. Predicting future elements of a sequence based on historical data is crucial for tasks ranging from language modeling to autonomous control.

A key challenge in sequence prediction is understanding the role of {\it context length}—the number of previous tokens used to make the upcoming prediction—and designing predictors that perform well with limited context due to computational and memory constraints. These resource constraints become particularly significant during the training phase of a predictor, where the computational cost of using long sequences can be prohibitive. Consequently, it is beneficial to design predictors that can learn from a smaller context length while still generalizing well to longer sequences. This leads us to the central question of our investigation: Can we develop algorithms that learn effectively using short contexts but perform comparably to models that use longer contexts?

To address this question, we introduce a new performance metric—{\unfairregrettext}—which measures the difference in total prediction loss between an online predictor with limited context length and a benchmark predictor with a longer context. Unlike classical regret, which assumes both the learner and the benchmark operate under the same conditions, {\unfairregrettext} accounts for the asymmetry in context lengths, providing a more realistic assessment of performance in resource-constrained settings. With a formal and well-defined notion of {\unfairregrettext} in hand, we begin our investigation with the following question: are there algorithms that can attain non-trivial bounds on the {\unfairregrettext} for natural sequences? 

We explore this concept through the lens of spectral filtering algorithms \citep{hazan2017learning,hazan2018spectral}. Spectral filtering has emerged as a robust method for learning linear dynamical systems when the system is unknown and the hidden state is unobserved. Beyond their theoretically sound properties, spectral filtering-based predictors have proven practical in recent applications. Notably, the Spectral Transform Unit \citep{agarwal2023spectral}, a neural architecture built using spectral filtering, has recently shown promise on sequence prediction over a range of modalities \citep{liu2024flash}.

In this work, we extend the theoretical understanding of spectral filtering by demonstrating that these predictors can achieve length generalization. Specifically, we present a gradient-based online learning algorithm for spectral filtering and show that we can train on a smaller context length while still achieving the same regret bounds as if we had trained on a longer context length. Formally, we prove that this algorithm guarantees sublinear {\unfairregrettext}, indicating that the performance gap diminishes as the sequence length increases.

Beyond theoretical interest, our work is practically motivated by challenges in length generalization faced by large language models (LLMs). Current LLMs often struggle to generalize to longer sequences than those seen during training \citep{abbe2023generalization, anil2022exploring, jelassi2023length, zhou2023algorithms, deletang2022neural, dziri2024faith, zhou2024transformers} and a significant body of empirical research has been dedicated to addressing this limitation \citep{kazemnejad2024impact, shen2023positional, dai2019transformer, chi2022kerple, li2023functional, press2021train}. Despite its importance and extensive empirical research, provable theoretical results on length generalization remain largely elusive. We view our work as a step toward addressing this gap. More practically, most methods introduced to improve length generalization are task-specific. Our work suggests that neural architectures that incorporate spectral filtering, like the Spectral Transform Unit, have the potential to provide robust length generalization.

\subsection{Our Contributions}
Consider {\bf online sequence prediction} in which the predictor iteratively receives input $u_t \in \R^{d_{\textrm{in}}}$ and then makes a prediction $\hat{y}_t \in \R^{d_{\textrm{out}}}$ of the output, after which the true output $y_t$ is revealed. The goal of the predictor is to minimize error according to a given convex and Lipschitz loss function $\ell_t(y_t, \hat{y}_t)$. In this work we consider the class of \emph{spectral filtering} predictors, introduced by \citet{hazan2017learning}. A spectral filtering predictor is characterized by parameters $(T, {M_i}_{i=1}^k, k)$ and outputs predictions $\hat{y}_t$ of the form
\begin{equation*}
    \hat{y}_t = y_{t-1} + \sum_{i = 1}^k M_i u_{(t-1):0} \phi_i,
\end{equation*}
where $u_{(t-1):0} \in \mathbb{R}^{d_{\textrm{in}} \times T}$ is a matrix whose columns are the previous inputs $u_{t-1}, u_{t-2}, \dots, u_0$ (possibly zero-padded as necessary), $\{\phi_j\}_{j=1}^k$ are the $T$-dimensional spectral filters, $\left \{ M_i \right \}_{i = 1}^k \subset \R^{d_{\textrm{out}} \times d_{\textrm{in}}}$ are matrices which are learned online, and $k$ is the number of filters used.  \citet{hazan2017learning} provide an algorithm to learn $\left \{ M_i \right \}_{i = 1}^k$ and show this achieves nearly optimal regret bounds when measured against the best Linear Dynamical System (LDS) predictor. We investigate whether it is necessary to use the entire history $u_{(t-1):0}$ to learn the optimal set of matrices $\left \{ M_i \right \}_{i = 1}^k$. More broadly, we explore whether predictor classes and corresponding online learning algorithms exist that can achieve context length generalization—that is, they use only a short recent history during learning but perform nearly as well as if they had used the full, much longer history length. Of course, predictors which perform poorly on systems that require long memory can trivially achieve context length generalization if their performance is poor regardless of the context length used. Therefore, it is important to note that one of the key features of spectral filtering predictors is that they are able to perform well on systems that have long memory \citep{hazan2017learning}.

To properly understand context length generalization, we introduce the notion of \emph{\unfairregrettext}. The idea is to consider the regret of learning a predictor from a class which is only allowed to use context length $L'$ against the best predictor which is allowed to use (potentially much longer and therefore asymmetric) context length $L$. Let $\prod_{L}$ denote the class of predictors in $\prod$ which use context length $L$. Given an algorithm $\mA(L')$ which learns over predictors from some class $\prod_{L'}$, the {\unfairregrettext} over horizon $T$ is
\begin{equation*}
    \unfairregret \left( \mA (L'), \prod_L \right) \defeq \sum_{t = 1}^T \ell_t(y_t, \hat{y}_t^{\mA(L')}) - \min_{\pi \in \prod_{L}} \ell_t(y_t, \hat{y}_t^{\pi}).
\end{equation*}
Our main result shows that spectral filtering generalizes from a history of $T^q$, where $q \in [0,1]$, to $T$ for certain linear dynamical systems. It is formally given in the following theorem.

\begin{theorem}
\label{thm:lengthgeneralizationintro}
Let $T \in \mathbb{Z}_{\geq 0}$ and $q \in [0, 1]$. Consider a sequence $(y_1, \dots, y_T)$ generated by an unknown and noiseless linear dynamical system defined by matrices $(A,B,C,D)$ as per Eq.~\ref{eqn:lds_equations}. Assume the input sequence $u_{0:(t-1)}$ is sufficiently well-conditioned, satisfying $\sum_{t = 0}^{T-1} (T - t) u_t u_t^{\top} \succeq \left( \frac{2 |C| |B|}{\sqrt{T}} \right) I$. Suppose the eigenvalues of $A$ lie within the range $\left[ 0, 1 - \frac{\log(T)}{8T^q} \right] \cup \left[ 1 - \frac{1}{2T^{5/4}}, 1 \right]$.

Let $\mathcal{A}(L)$ denote Algorithm~\ref{alg:ogd_short_length} operating with context length $L$, and let $\prod^{\mathrm{SF}}_L$ denote the class of spectral filtering predictors using context length $L$. For the squared loss $\ell_t(y, y') = | y - y' |^2$ and sufficiently large $T$, it holds that:
\begin{equation*}
    \unfairregret \left( \mA(T^q), \prod^{\mathrm{SF}}_{T} \right) \leq \tilde{O}(\sqrt{T}).
\end{equation*}
\end{theorem}

This theorem indicates that for any $q \in [0,1]$, the {\unfairregrettext} is bounded by $\tilde{O}(\sqrt{T})$. However, as $q$ decreases, the class of linear dynamical systems for which this bound holds becomes more restricted due to the eigenvalue conditions on $A$. The spectrum of $A$ determines the memory of the system; when the eigenvalues of $A$ are $1$, the system is only marginally-stable and standard predictors which aim to use low memory typically fail. 
Critically, Theorem~\ref{thm:lengthgeneralizationintro} holds even for these marginally-stable systems. When interpreting this result, it's important to note that the class of spectral filtering predictors $\prod^{\mathrm{SF}}_{T}$ which use the full context length are provably able to predict well on marginally-stable Linear Dynamical Systems \citep{hazan2017learning}\footnote{The only LDS's for which there can be any useful results are those with $A$'s eigenvalues in $[-1, 1]$, i.e. marginally-stable systems. We recall that the spectral filtering principle can be readily applied to handle negative eigenvalues in $[-1, 0]$ (see Appendix D of \cite{agarwal2023spectral}, for example). For ease of presentation, we focus on capturing the length generalization effects of eigenvalues in $[0, 1]$ in the sequel, and so we suppose without loss of generality that $A \succeq 0$.}. Therefore, this result implies that spectral filtering predictors are able to context length generalize in a nontrivial way.

Inspired by particular spectrum of $A$ that is required for the classical Spectral Filtering algorithm to achieve length generalization, we develop a novel variation on the Spectral Filtering algorithm, presented in Algorithm~\ref{alg:sf_two}, which achieves length generalization without added assumptions on the spectrum of $A$ (whenever the context-length is at least $T^{1/3}$). Algorithm~\ref{alg:sf_two} achieves this by using two autoregressive components $y_{t-1}$ and $y_{t-2}$ to construct its prediction $\hat{y}_t$ of $y_t$. We provide the following theoretical result. 
\begin{theorem}
\label{thm:lengthgeneralizationintro_two}
Let $T \in \mathbb{Z}_{\geq 0}$ and $q \in [0, 1]$. Consider a sequence $(y_1, \dots, y_T)$ generated by an unknown and noiseless linear dynamical system defined by matrices $(A,B,C,D)$ as per Eq.~\ref{eqn:lds_equations}. Assume the input sequence $u_{0:(t-1)}$ is sufficiently well-conditioned, satisfying $\sum_{t = 0}^{T-1} (T - t) u_t u_t^{\top} \succeq \left( \frac{2 |C| |B|}{\sqrt{T}} \right) I$. Suppose the eigenvalues of $A$ lie within the range $\left[ 0, 1 - \frac{\log(T)}{8T^q} \right] \cup \left[ 1 - \frac{1}{2T^{1/4}}, 1 \right]$.

Let $\mathcal{A}(L)$ denote Algorithm~\ref{alg:sf_two} operating with context length $L$, and let $\prod^{\mathrm{SF}}_L$ denote the class of spectral filtering predictors using context length $L$. For the squared loss $\ell_t(y, y') = | y - y' |^2$ and sufficiently large $T$, it holds that:
\begin{equation*}
    \unfairregret \left( \mA(T^q), \prod^{\mathrm{SF}}_{T} \right) \leq \tilde{O}(\sqrt{T}).
\end{equation*}
\end{theorem} 
Observe that if $q \geq 1/3$, then $[0, 1 - \log(T)/(8T^q)] \cup [1-1/T^{1/4}, 1] = [0,1]$ for any $T>0$ and so we do not constrain the spectrum of $A$ to get length generalization (aside from assuming it has nonnegative eigenvalues).

Our next contribution is the development of a new class of predictors we call tensorized spectral filters. Tensorized spectral filters possess more structure than their original counterparts and are provably more expressive—they can learn a select class of time-varying linear dynamical systems that vanilla spectral filtering cannot. We develop a novel context-length dependent algorithm for tensorized spectral filtering which, similar to Algorithm~\ref{alg:sf_two}, requires two autoregressive components.

\begin{theorem}
\label{thm:lengthgeneralizationintro_tensors}
Let $T \in \mathbb{Z}_{\geq 0}$ and $q \in [0, 1]$. Consider a sequence $(y_1, \dots, y_T)$ generated by an unknown and noiseless linear dynamical system defined by matrices $(A,B,C,D)$ as per Eq.~\ref{eqn:lds_equations}. Assume the input sequence $u_{0:(t-1)}$ is sufficiently well-conditioned, satisfying $\sum_{t = 0}^{T-1} (T - t) u_t u_t^{\top} \succeq \left( \frac{2 |C| |B|}{\sqrt{T}} \right) I$. Suppose the eigenvalues of $A$ lie within the range $\left[ 0, 1 - \frac{\log(T)}{8T^q} \right] \cup \left[ 1 - \frac{1}{2T^{1/4}}, 1 \right]$.
Let $\mathcal{A}(L)$ denote Algorithm~\ref{alg:tensor_sf} operating with context length $L$, and let $\prod^{\mathrm{SF}}_L$ denote the class of spectral filtering predictors using context length $L$. For the squared loss $\ell_t(y, y') = | y - y' |^2$ and sufficiently large $T$, it holds that:
\begin{equation*}
    \unfairregret \left( \mA(T^q), \prod^{\mathrm{SF}}_{T} \right) \leq \tilde{O}(\sqrt{T}).
\end{equation*}
\end{theorem}

Finally, we experimentally confirm the results of Theorem \ref{thm:lengthgeneralizationintro} on synthetic data generated by an LDS. Interestingly, we find that Theorem~\ref{thm:lengthgeneralizationintro} accurately predicts when length generalization is possible; indeed, when the data is generated by an LDS which has eigenvalues in the ``bad'' range $[1 - \log(T)/(8T^q), 1-1/(2T^{5/4})]$ we find that the limited context length spectral filtering predictors are unable to length generalize. However, when the data is generated by and LDS which has eigenvalues ``hugging'' this bad range (i.e. either just smaller than $1 - \log(T)/(8T^q)$ or just larger than $1-1/(2T^{5/4})$), the limited context length spectral filtering predictors successfully length generalize. Next, we conduct experiments using the STU neural architecture to test the hypothesis that this architecture should simply length generalize without any task-specific engineering. We consider the induction heads synthetic task and find that the out-of-the-box STU neural architecture does indeed enjoy some level of length generalization. This suggests that incorporating spectral filtering into neural architectures, like the STU, may provide improved length generalization in deep learning applications. We leave further empirical study on this for future work.

\subsection{Related Work}

The literature for sequence prediction is too broad to survey in detail, so we give a few highlights of the recent rapid advancements. The most notable progress includes the Transformer model \citep{vaswani2017attention} that incorporates an attention mechanism for accurate sequence prediction in many domains \citep{brown2020language, dosovitskiy2020image, jumper2021highly}. Transformer models and their attention layers have memory/computation requirements that scale quadratically with context length. Many approximations have been proposed (see \cite{10.1145/3530811} for a recent survey). 

Motivated by the high memory and compute requirements of transformers, state space models were revisited starting from \citep{NEURIPS2020hippo,gu2021combining} who propose and develop the HiPPO theory.  \citet{gu2021efficiently} develop the S4 parameterization to address the bottlenecks of training efficiency, performance and numerical stability. Further works in the area show SOTA performance and include  \cite{gupta2022diagonal,smith2023simplified,orvieto2023resurrecting,gu2023mamba}. 

State space models are very efficient for training and inference, but can suffer in long-context applications. This motivated the use of spectral filtering technique for learning marginally-stable linear dynamical systems \citep{hazan2017learning,hazan2018spectral}.  This technique was incorporated to a neural architecture in \cite{agarwal2023spectral}, that was recently shown to perform well across several modalities \citep{liu2024flash}.

From an applied perspective, generalization in sequence prediction has been recently studied in \cite{hou2024universal} through the theoretical lens of Turing programs. They propose a methodology that empirically improves length generalization across a diverse set of tasks. There are also architecture-specific approaches to length generalization such as ALiBi positional embeddings for transformers \citep{press2022trainshorttestlong}, but such methods lack provable guarantees and can have varying empirical performance \citep{kazemnejad2024impact}.

In contrast, our investigation starts from the theory of regret minimization in games and online learning. Regret minimization has the advantage that it implies generalization in the statistical learning setting (see e.g. \cite{cesa2004generalization}) and is usually accompanied by efficient algorithms such as online gradient descent (see e.g. \cite{hazan2016introduction}). Our new notion of {\unfairregrettext} incorporates asymmetric information access between the online learner and the benchmark class.

\section{Background and Setting}

In the {\bf online sequence prediction} setting the predictor iteratively receives input $u_t$ and makes prediction $\hat{y}_t$ of the output, after which the true output $y_t$ is revealed. The goal is to minimize error according to a given (convex Lipschitz) loss function $\ell_t(y_t, \hat{y}_t)$.

In online learning, we usually do not make statistical assumptions about the generation of the input sequence. As such, performance is measured relative to a certain benchmark class of predictors. For example, a linear predictor predicts according to the rule 
$$ \pi_{M_{1:k},N_{1:l}} (u_{0:(t-1)},y_{1:t-1}) = \sum_{i=1}^k M_i u_{t-i} + \sum_{j=1}^l N_j y_{t-j}.$$
A prediction algorithm $\mA$ is measured by regret, or difference in total loss, vs. a class of predictors $\prod$ (such as linear predictors), i.e.
$$ \regret_T(\mA,\prod) = \sum_{t=1}^T \ell_t( y_t , \hat{y}_t^\mA ) - \min_{\pi \in \prod} \sum_{t=1}^T \ell_t( y_t , \hat{y}_t^\pi ). $$

This formulation is valid for online sequence prediction of any signal. 
We are particularly interested in signals that are generated by dynamical systems. 
A time-invariant linear dynamical system is given by the dynamics equations 
\begin{equation} \label{eqn:lds_equations}
x_{t+1} = A x_t + B u_t + w_t \  \ , \ \ y_{t+1} = C x_t + D u_t + \zeta_t , 
\end{equation}
where $x_t$ is the (hidden) state, $u_t$ is the input or control to the system, and $y_t$ is the observation. The terms $w_t, \zeta_t$ are noise terms, and the matrices $A,B,C,D$ are called the system matrices. 
A linear dynamical predictor with parameters $A,B,C,D$ predicts according to 
\begin{equation*} 
\pi_{ABCD} (u_{0:(t-1)},y_{1:t-1}) = \hat{y}_t^{\pi} = \sum_{i=1}^t C A^i B u_{t-i} + D u_t .  
\end{equation*}
The best such predictor for a given sequence is also called the optimal open loop predictor, and it is accurate if the signal is generated by a LDS without noise.

\subsection{Context Length Generalization and the {\unfairregrettext} metric}

We say that an online predictor has context length $L$ if it bases its prediction $\hat{y}_t$ only on information from the previous $L$ timesteps, i.e. $u_{t:t-L}$ and $y_{t:t-L}$. Open loop predictors base their prediction only on $u_{t:t-L}$, whereas closed loop predictors can also use $y_{t:t-L}$. 

For example, the class of all linear open loop predictors with context lengths $L$ is given by
$$ \prod_L^{OL} = \left\{ \pi_{M_{1:L}}\; |\; \pi_{M_{1:L}}(u_{(t-1):0})   = \sum_{i=1}^L M_i u_{t-i} \right\} .  $$

The key question in our work is whether there are predictor classes with corresponding online learning algorithms which context length generalize in the sense that they learn the best predictor in the class using a short context length, but they perform well compared to the best predictor which is allowed to use long context length. To formalize this notion, we introduce {\unfairregrettext} whose definition we restate here:
\begin{definition}[\unfairregrettext]
    Let $\prod_{L'}^{\textrm{learn}}$ be a class of predictors which use context length $L'$ and let $\prod_{L}^{\textrm{ref}}$ be a reference class of predictors which use context length $L$. The \emph{{\unfairregrettext}} with respect to (convex Lipschitz) loss $\ell_t$ over horizon $T$ of an algorithm $\mA(L')$ which tries to learn a predictor from $\prod_{L'}^{\textrm{learn}}$ is
    \begin{equation*}
        \unfairregret \left( \A(L'), \prod_{L}^{\textrm{ref}} \right) \defeq \sum_{t=1}^T \ell_t( y_t , \hat{y}_t^{\mA(L')} ) - \min_{\pi \in \prod_L} \sum_{t=1}^T \ell_t( y_t , \hat{y}_t^\pi ).
    \end{equation*}
\end{definition}

\subsection{Spectral Filtering}

Spectral filtering is a notable deviation from the standard theory of linear dynamical systems that allows efficient learning in the presence of arbitrarily long memory \citep{hazan2017learning}. The idea is to project the sequence of inputs to a small subspace that is constructed using the special structure of discrete linear dynamical systems. The output of the spectral filtering predictor is represented  as 
\begin{equation}
\label{eqn:SFbasic}
\hat{y}_{t} = y_{t-1} + \sum_{i=1}^k M_{i} {u}_{(t-1):0}  \phi_i,
\end{equation}
where $u_{(t-1):0} \in \mathbb{R}^{d_{\textrm{in}} \times T}$ is a matrix whose columns are the previous inputs $u_{t-1}, \dots, u_0$ (possibly zero-padded as necessary), $\{\phi_j\}_{j=1}^k$ are the $T$-dimensional spectral filters that can be computed offline given the target sequence length $T$, and $\left \{ M_i \right \}_{i = 1}^k \subset \R^{d_{\textrm{out}} \times d_{\textrm{in}}}$ are the matrices parameterizing the model. These spectral filters are the eigenvectors of the matrix constructed as the average of outer products of the discrete impulse-response functions as we now detail.

Let $\mu_{\alpha,T} = (1-\alpha)[1 , \alpha, \alpha^2 , ..., \alpha^T]$ be the (weighted) impulse-response vector corresponding to a one dimensional linear dynamical system with parameter $\alpha$ unfolded to $T$ time steps, and consider the symmetric matrix 
\begin{equation}
    \label{eqn:Hankel_def}
    H_T \defeq \int_{0}^1 \mu_{\alpha,T} \mu_{\alpha,T}^{\top} d \alpha.
\end{equation}
Since $H_T$ is a real PSD matrix, it admits a real spectral decomposition, and the (non-negative) eigenvalues can be ordered naturally by their value. Let 
$ \{(\sigma_j \in \reals, \phi_j \in \reals^{L})\}_{j=1}^{L}$ be the eigenvalue-eigenvector pairs of $H_T$ ordered to satisfy $\sigma_1 \geq \sigma_2 \geq \ldots \geq \sigma_d$. The spectral filters $\phi_1,...,\phi_k$ are exactly those first $k$ eigenvectors corresponding to the largest eigenvalues. 
The spectral filtering class is further parameterized by matrices  $M_1,...,M_k \in \reals^{\dout \times \din}$. The output at time $t$ is then given by equation \eqref{eqn:SFbasic}.

The following theorem establishes that the spectral filtering class of predictors approximately contains bounded linear dynamical systems with positive semi-definite $A$. The exact constants are left out for simplicity of presentation, but appear in the original work. 

\begin{theorem}[Simplified from \cite{hazan2017efficient}]
\label{thm:hszthm}
Given any linear dynamical system parametrized by $A,B,C,D$ such that $A$ is a PSD matrix with $\|A\| \leq 1$, there exists matrices $M_1,...,M_K$, such that for all $L$ and all sequences $u_{1:L}, \|u_t\| \leq 1$, the following holds. Let $y^{\mathrm{LDS}}_{1:L}$ be the sequence generated by execution of the LDS via \eqref{eqn:lds_equations} and $y^{\mathrm{SF}}_{1:L}$ be the sequence generated by Spectral Filtering via \eqref{eqn:SFbasic}. Then for all $t \in [L]$,
\[ \|y^{\mathrm{LDS}}_{t} - y^{\mathrm{SF}}_{t}\| \sim e^{-  \frac{k}{\log(L)} } . \]
\end{theorem}

Theorem~\ref{thm:hszthm} establishes that Spectral Filtering can predict long memory sequences since the statements holds even over marginally stable linear dynamical systems.

\section{Learning with a Short Context—Provable Length Generalization for Linear Dynamical Systems}

In Algorithm~\ref{alg:ogd_short_length}, we modify the classical online learning algorithm for spectral filtering to use a shorter context window. To properly define our notion of length generalization, we need to distinguish between context lengths. Thus we introduce the notation for the loss observed with a context length $L$:
\begin{equation*}
    \ell_t(M, L)  \defeq \| y_t - y_{t-1} - \sum_{i=1}^k M_{i}  u_{(t-1):(t-L)} \phi_i \|^2,
\end{equation*}
where $\left \{ M_i \right \}_{i = 1}^k$ and $\left \{ \phi_i \right \}_{i = 1}^k$ are as defined for Eq.~\ref{eqn:SFbasic} and $u_{(t-1):(t-L)} \in \mathbb{R}^{d_{\textrm{in}} \times T}$ is the matrix $u_{(t-1):0}$ but with the columns corresponding to the inputs after context length $L$, i.e. $u_{t - L - 1}, \dots, u_0$, zeroed out.  
Note that this is overloaded notation compared with $\ell_t(y,y')$ which measures the loss of the true $y$ with the predicted $y'$ as used in our definition of regret. The context length specific loss can be written equivalently as
\begin{equation*}
     \ell_t(M, L) =  \| \hat{y}(M^t, L) - y_t \|^2,
\end{equation*}
where $\hat{y}(M^t, L)$ denotes the prediction of $y_t$ using iterate $M^t$ and context window size $L$ as in Eq.~\ref{eqn:shalom4} of Algorithm~\ref{alg:ogd_short_length}.
\begin{algorithm}[ht]
\caption{Spectral Filtering with Limited Context} \label{alg:ogd_short_length}
\begin{algorithmic}[1]
\STATE {\bf Input:} $k > 0, T >0$, $L > 0$, $r > 0$. Initialize $M_{i}^1 \in \R^{d_{\textrm{out}} \times d_{\textrm{in}}}$ for $i \in [k]$ and set $M^1 = [M_1^1, \dots, M_k^1]$. Let $\phi_{1:k}$ be the largest eigenvectors of $H_T$ defined in Eq.~\ref{eqn:Hankel_def} with corresponding eigenvalues $\sigma_{1:k}$, and let $\pi_{\K}(\cdot)$ denote the projection to convex set $\K$.
\FOR {$t = 1,2,...,T$}
\STATE Compute and predict  
\begin{equation} \label{eqn:shalom4}
\hat{y}_t = y_{t-1} + \sum_{i =1}^{k} M_{i}^t u_{(t-1):(t-L)}  (\sigma_i^{1/4} \phi_{i}) . 
\end{equation}
\STATE Observe $y_t$, denote $\ell_t(M^t,L) = \|\hat{y_t} - y_t\|^2$ and update and project onto the low Frobenius norm ball
$$ \hat{M}^{t+1} \leftarrow M^{t} - \eta_t \nabla_{M} \ell_t( M^t ) $$
$$ M^{t+1} = \pi_{\K} \left( \hat{M}^{t+1} \right), $$
where $\K_{\mtrueconst} = \left\{ M \in \reals^{k \times d_{\textrm{out}} \times d_{\textrm{in}}} \textrm{ s.t. } \norm{M_{i}} \leq \mtrueconst  \textrm{ for all } i \in [k] \right\}$.
\ENDFOR 
\end{algorithmic}
\end{algorithm}

To provide a precise statement on length generalization, we present the following performance guarantee. Note that we measure loss differently for our prediction and the benchmark predictor—the comparator has access to more information. Therefore, this guarantee is stronger and more challenging to obtain than classical regret bounds. Note that we prove the following for a noiseless $(A,B,C,I)$-LDS rather than $(A,B,C,D)$ which is without loss of generality since we can consider the input as $Du_1, \dots, Du_T$. 
\begin{theorem}
\label{thm:lengthgeneralization}
Let $T \in \mathbb{Z}_{\geq 0}$ and $q \in [0, 1]$. Consider a sequence $(y_1, \dots, y_T)$ generated by an unknown and noiseless linear dynamical system defined by matrices $(A,B,C,I)$ as per Eq.~\ref{eqn:lds_equations}. Assume the input sequence $u_{0:(t-1)}$ is sufficiently well-conditioned, satisfying $\sum_{t = 0}^{T-1} (T - t) u_t u_t^{\top} \succeq \left( \frac{2 |C| |B|}{\sqrt{T}} \right) I$. Suppose the eigenvalues of $A$ lie within the range $\left[ 0, 1 - \frac{\log(T)}{8T^q} \right] \cup \left[ 1 - \frac{1}{2T^{5/4}}, 1 \right]$. Let $k  = \Omega \left(  \log(T) \cdot \log \left( T  d_A \right) \right)$, $r \geq \norm{B} \norm{C}$, and assume $T \geq (4 k  \log(T)/ \norm{C} \norm{B} )^{4}$. Algorithm~\ref{alg:ogd_short_length} satisfies: 
$$     \unfairregret\left( \mA(T^q), \prod^{\mathrm{SF}}_{T} \right) = \sum_{t=1}^T \ell_t( M^t, T^q) - \min_{M^* \in \K_{r}} \sum_{t=1}^T \ell_t(M^* , T) \leq  O \left( \norm{B}^2 \norm{C}^2 k^{3/2} \log(T) \sqrt{T} \right).$$
\end{theorem}
The proof of Theorem~\ref{thm:lengthgeneralization} is in Appendix~\ref{appendix:vanilla_proof} with a high-level overview in Section~\ref{subsection:high_level_proof}. This theorem shows that the sequence $M^1, \dots, M^T$ constructed by Algorithm~\ref{alg:ogd_short_length}, even when using a reduced context length of size $T^q$, is able to achieve regret $O(\sqrt{T})$ when compared to the best spectral filter that uses full context length $T$.

To better understand the eigenvalue ranges not covered by Theorem~\ref{thm:lengthgeneralization}, Figure~\ref{fig:eigvalintervals} highlights the regions of $[0, 1]$ that are excluded from the theorem's guarantee for different context lengths $T^q$. The $x$-axis represents $T$ on the x-axis and the unfavorable subsets of $[0, 1]$ are depicted as vertical slices. As $T$ and $q$ increase, our coverage of $[0, 1]$ improves.
Notably, our method applies even to marginally stable systems that exhibit long memory due to eigenvalues  $\alpha \in [1 - 1/(2T^{5/4}), 1]$.

\begin{figure}[ht]
    \centering
\includegraphics[width=0.9\textwidth]{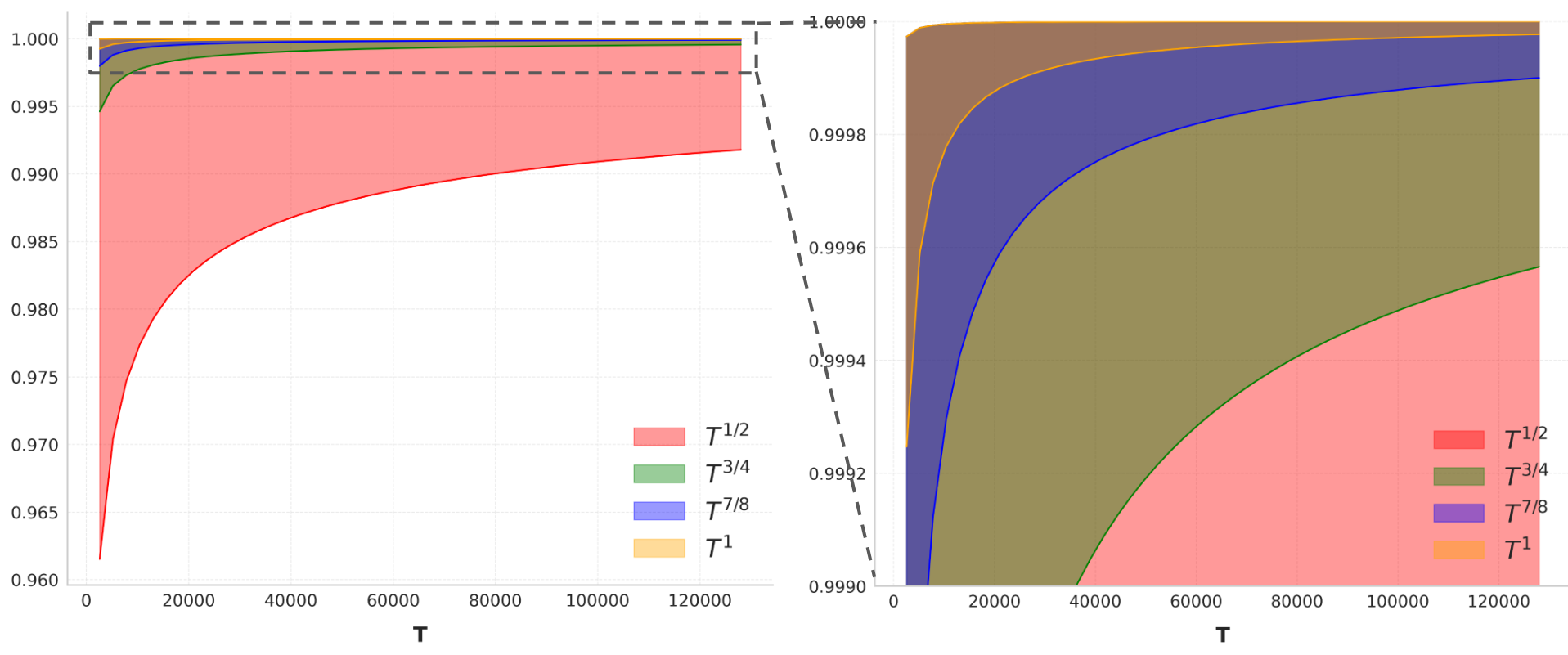}
    \caption{Regions of $[0, 1]$ not covered by Theorem~\ref{thm:lengthgeneralization}, with $T$ on the x-axis. For convenience, in the right image we zoom in to $[0.999, 1]$.}
    \label{fig:eigvalintervals}
\end{figure}

Motivated by the limitations of Theorem~\ref{thm:lengthgeneralization} to provide a length generalization that is robust to the spectrum of $A$, we introduce a variation on the classical Spectral Filtering algorithm, presented as Algorithm~\ref{alg:sf_two}. This algorithm uses the two most previous outputs $y_{t-1}$ and $y_{t-2}$ when making prediction $\hat{y}_t$ of $y_t$. 

This algorithm has a slightly different construction of spectral filters. Indeed, they are the eigenvectors of the following matrix
\begin{equation}
\label{eqn:N_T_def}
    N_T \defeq \int_{0}^1 \tilde{\mu}_{\alpha,T} \tilde{\mu}_{\alpha,T}^{\top} d \alpha,
\end{equation}
where $\tilde{\mu}_{\alpha,T} \defeq (1 - \alpha)^2 [1, \alpha, \alpha^2, \dots, \alpha^T]$. Interestingly, just by using one extra autoregressive term, our adapted algorithm is able to enjoy \emph{robust} length generalization in the sense that whenever the context window is at least $T^{1/3}$ then no extra assumptions on the spectrum of $A$ are necessary to achieve our notion of length generalization. We state this formally in the following theorem.  
\begin{algorithm}[ht]
\caption{Spectral Filtering with Limited Context and Two Autogressive Components} \label{alg:sf_two}
\begin{algorithmic}[1]
\STATE {\bf Input:} $k > 0, T >0$, $L > 0$, $r > 0$. Initialize $M_{i}^1 \in \R^{d_{\textrm{out}} \times d_{\textrm{in}}}$ for $i \in [k]$ and set $M^1 = [M_1^1, \dots, M_k^1]$. Let $\tilde{\phi}_{1:k}$ be the largest eigenvectors of $N_{T-2}$ defined in Eq.~\ref{eqn:N_T_def} with corresponding eigenvalues $\tilde{\sigma}_{1:k}$, and let $\pi_{\K}(\cdot)$ denote the projection to convex set $\K$.
\FOR {$t = 1,2,...,T$}
\STATE Compute and predict  
\begin{equation*} 
\hat{y}_t = 2 y_{t-1} - y_{t-2} + M_1^t u_{t-1} + M_2^t u_{t-2} +  \sum_{i =3}^{k} M_{i}^t u_{(t-3):(t-L)}  (\tilde{\sigma}_i^{1/4} \tilde{\phi}_{i}) . 
\end{equation*}
\STATE Observe $y_t$, denote $\ell_t(M^t,L) = \|\hat{y_t} - y_t\|^2$ and update and project onto the low Frobenius norm ball
$$ \hat{M}^{t+1} \leftarrow M^{t} - \eta_t \nabla_{M} \ell_t( M^t ) $$
$$ M^{t+1} = \pi_{\K} \left( \hat{M}^{t+1} \right), $$
where $\K_{\mtrueconst} = \left\{ M=[M_1, \dots, M_k] \textrm{ s.t. } \norm{M_{i}} \leq \mtrueconst  \textrm{ for all } i \in [k] \right\}$.
\ENDFOR 
\end{algorithmic}
\end{algorithm}

\begin{theorem}
\label{thm:lengthgeneralization_two_auto}

Let $T \in \mathbb{Z}_{\geq 0}$ and $q \in [0, 1]$. Consider a sequence $(y_1, \dots, y_T)$ generated by an unknown and noiseless linear dynamical system defined by matrices $(A,B,C,I)$ as per Eq.~\ref{eqn:lds_equations}. Assume the input sequence $u_{0:(t-1)}$ is sufficiently well-conditioned, satisfying $\sum_{t = 0}^{T-1} (T - t) u_t u_t^{\top} \succeq \left( \frac{2 |C| |B|}{\sqrt{T}} \right) I$. Suppose the eigenvalues of $A$ lie within the range $\left[ 0, 1 - \frac{\log(T)}{8T^q} \right] \cup \left[ 1 - \frac{1}{2T^{1/4}}, 1 \right]$. Let $k  = \Omega \left(  \log(T) \cdot \log \left( T  d_A \right) \right)$, $r \geq \norm{B} \norm{C}$ and assume $T \geq (4 k  \log^2(T)/ \norm{C} \norm{B} )^{4}$. Algorithm~\ref{alg:sf_two} satisfies: 
$$     \unfairregret\left( \mA(T^q), \prod^{\mathrm{SF}}_{T} \right) = \sum_{t=1}^T \ell_t( M^t, T^q) - \min_{M^* \in \K_{r}} \sum_{t=1}^T \ell_t(M^* , T) \leq  O \left( \norm{B}^2 \norm{C}^2 k^{3/2} \log^2(T) \sqrt{T} \right).$$
\end{theorem}
The proof of Theorem~\ref{thm:lengthgeneralization_two_auto} is in Appendix~\ref{appendix:two_vanilla} and we give a high-level overview of the proof in the following section. Observe that if $q \geq 1/3$, then $[0, 1 - \log(T)/(8T^q)] \cup [1-1/T^{1/4}, 1] = [0,1]$ for any $T>0$ and so we do not constrain the spectrum of $A$ to get length generalization.

\subsection{High Level Proof Overview of Main Theorems}
\label{subsection:high_level_proof}
The general proof technique for both Theorem~\ref{thm:lengthgeneralization} and Theorem~\ref{thm:lengthgeneralization_two_auto} is the same. First, using standard online gradient descent results we prove that the iterates $M^t$ achieve $O(\sqrt{T})$ regret as measured by the context-length restricted loss $\sum_{t = 1}^T \ell_t(M, L)$. That is,
\begin{equation}
\label{eqn:ogd_regret}
    \sum_{t = 1}^T \ell_t(M^t, L) \leq \min_{M \in \K_{\mtrueconst}} \sum_{t = 1}^T \ell_t(M, L) + O(\sqrt{T}).
\end{equation}
Next we prove that the (unique) $M^*$ which minimizes the loss on the full $T$-length context achieves length generalization in the sense that it achieves small loss even when only allowed to use context length $L$. That is, defining
\begin{equation*}
    M_T^* \defeq \argmin_{M \in \K_{\mtrueconst}} \sum_{t = 1}^T \ell_t(M, T),
\end{equation*}
we have
\begin{equation}
\label{eqn:length_gen_MT}
    \sum_{t = 1}^T \ell_t(M_T^*, L) \leq \sum_{t = 1}^T \ell_t(M_T^*, T) + O(\sqrt{T}).
\end{equation}
Since it is trivially the case that
\begin{equation*}
    \min_{M \in \K_{\mtrueconst}} \sum_{t = 1}^T \ell_t(M, L) \leq \sum_{t = 1}^T \ell_t(M_T^*, L),
\end{equation*}
we can combine Eq.~\ref{eqn:ogd_regret} and Eq.~\ref{eqn:length_gen_MT} to get the final notion of length generalization that
\begin{equation*}
    \sum_{t = 1}^T \ell_t(M^t, L) \leq  \min_{M \in \K_{\mtrueconst}} \sum_{t = 1}^T \ell_t(M, L) + O(\sqrt{T}) \leq \sum_{t = 1}^T \ell_t(M_T^*, L) + O(\sqrt{T}) \leq  \sum_{t = 1}^T \ell_t(M_T^*, T) + O(\sqrt{T}).
\end{equation*}
The difficult result to prove is Eq.~\ref{eqn:length_gen_MT}.  The high level idea is that when $y_{1:t}$ evolves as a noiseless LDS and when the input $u_{0:(t-1)}$ is sufficiently well-conditioned, then the minimizer for $\sum_{t = 1}^T \ell_t(M, T)$ approximately recovers a collection of ``true'' matrices which are generated by the underlying linear dynamical system. The second key idea is that if an algorithm had access to these ``true'' matrices then it would be able to achieve small loss even when restricted to a small context-length $L \ll T$. The extent to which these recovered matrices can achieve small loss when restricted to the small context-length depends on the way the algorithm chooses to predict $y_t$. In the case of Algorithm~\ref{alg:ogd_short_length} where $y_t$ is predicted based only using only one autoregressive term, even having access to the true matrices is not enough to accurately predict $y_t$. However, in the case of Algorithm~\ref{alg:sf_two}, having access to the true matrices as well as a second autoregressive term allows accurate prediction of $y_t$ even when restricted to small context-length window.

\section{Tensorized Spectral Filtering}

Adding some form of regularization to a prediction model often results in provably and empirically better generalization in many different problem settings. While length generalization differs from overall generalization, we explore how introducing regularization to spectral filters might enhance length generalization. To this end, we introduce the class of \emph{tensorized spectral filters}.

\begin{definition}[Tensorized Spectral Filters]  Let $\phi_{1}^d, \dots, \phi_d^d$ be the eigenvectors of the $d \times d$ Hankel matrix $H_d$ as defined in Eq.~\ref{eqn:Hankel_def}.  The class of $(d_1, d_2)$-dimensional \textbf{tensorized spectral filters} takes the form
\begin{equation*}
    \left \{ \phi_i^{d_1} \otimes \phi_j^{d_2} | i \in [d_1], j \in [d_2]\right \},
\end{equation*}
where $\otimes$ denotes the tensor (Kronecker) product.
\end{definition}

Part of the motivation for considering tensorized spectral filters comes from the tensor structure of the $\mu_{\alpha,T}$ which defines the matrix $H_T$ in Eq.~\ref{eqn:Hankel_def}. Observe that for a given $\alpha$ we have,
\begin{equation*}
    \mu_{\alpha, L^2} = (1 - \alpha^L)^{-1} \cdot \mu_{\alpha, L} \otimes \mu_{\alpha^{L},L}
\end{equation*}
Therefore, up to rescaling, the set of $\mu_{\alpha, L^2}$ is contained in the set of tensors $\mu_{\alpha, L} \otimes \mu_{\beta, L}$: 
\begin{equation} \label{eqn:shalom2}
\left\{ \mu_\alpha^{L^2}  | \alpha \in [0,1] \right\} \subseteq \left \{  \mu_\alpha^L \otimes \mu_\beta^L |   \alpha,\beta \in [0,1] \right\}  . 
\end{equation}
This property ensures that we can approximate the spectral filtering algorithm with its tensor approximation as per Algorithm~\ref{alg:tensor_sf}. In fact, the tensor spectral filtering algorithm is more expressive, as equation \eqref{eqn:shalom2} hints. One can construct non-linear dynamical systems that can be approximated by tensorized spectral filters but not by spectral filtering itself. Such a system would have dynamics corresponding to the tensor 
$$ \mu_\alpha^L \otimes \mu_\beta^L ,$$
where $\beta \neq \alpha^L$, and thus imply changing dynamics every $L$ iterations. In the following theorem we formalize this intuition that tensorized spectral filters are just as capable at modeling linear dynamical systems. 

\begin{theorem} \label{thm:expressivity}
Given any linear dynamical system parametrized by $A,B,C,D$ such that $A$ is a PSD matrix with $\|A\| \leq 1$, there exists matrices $M', M'', \{ M_{ij}, i,j \in [k] \}$, such that for sequences $u_{1:T}$ where $\|u_t\| \leq 1$ the following holds. Let $y^{\mathrm{LDS}}_{1:T}$ be the sequence generated by execution of the LDS via \eqref{eqn:lds_equations}, and $y^{\mathrm{TSF}}_{1:T}$ be the sequence generated by Tensor Spectral Filtering via:
\begin{equation*} \label{eqn:tensor_representation}
{y}_t^{\mathrm{TSF}} - 2 {y}_{t-1}^{\mathrm{TSF}} + y_{t-2}^{\mathrm{TSF}} = M' u_{t} + M'' u_{t-1} + \sum_{i,j=1}^{k} M_{i j}  u_{t-2:1}  \cdot \phi_{i} \otimes \phi_j .
\end{equation*}
Then for all $t \in [T]$,
\[ \|y^{\mathrm{LDS}}_{t} - y^{\mathrm{TSF}}_{t}\| \sim e^{-  \frac{k}{\log(T)} } . \]
\end{theorem}

\subsection{Length Generalization with Tensorized Spectral Filtering}

The representation property from Theorem~\ref{thm:expressivity} gives rise to a novel algorithm (Algorithm \ref{alg:tensor_sf}). Notably, this algorithm is based on online gradient descent for spectral filtering and can be shown to give the same regret bounds as the original spectral filtering algorithm from \cite{hazan2017learning}.  

\begin{algorithm}[ht]
\caption{Tensorized Spectral Filtering } \label{alg:tensor_sf}
\begin{algorithmic}[1]
\STATE {\bf Input:} $T>0$, $k > 0, L >0$, $c > 0$. Initialize $M'^{1}, M''^{1}, M_{ij}^{1} \in \R^{d_{\textrm{out}} \times d_{\textrm{in}}}$ for $i, j\in [k]$ and set $M^1 = [M'^{1}, M''^{1}, M_{11}^{1}, \dots, M_{kk}^{1}]$. Let $T' = (\lceil \sqrt{T-2}\rceil)^2 + 2$. Let $\phi_{1:k}$ be the $k$ largest eigenvectors of $H_{\sqrt{T'-2}}$, and $\pi_{\K}(\cdot)$ denote the projection to set $\K_{\mtrueconst}$.
\FOR {$t = 1,2,...,T$}
\STATE Compute and predict  
\begin{equation*}
\hat{y}_t = 2 y_{t-1} - y_{t-2}  + M'^t u_{t-1} + M''^t u_{t-2} +  \sum_{i, j=1}^{k} M_{i j}^t  u_{t-3:t-L}(\sigma_i^{1/4} \phi_{i} \otimes \sigma_j^{1/4} \phi_j ), 
\end{equation*}
where $u_{t-3:t-L} \in \R^{d_{\textrm{in}} \times (T' - 2)} $ is padded appropriately.
\STATE Observe $y_t$, denote $\ell_t(M^t) = \|\hat{y_t} - y_t\|^2$ and update and project update and project onto the low Frobenius norm ball
$$ \hat{M}^{t+1} \leftarrow M^{t} - \eta_t \nabla_{M} \ell_t( M^t ) $$
$$ M^{t+1} = \proj_{\K} \left( \hat{M}^{t+1} \right), $$
where $\K_{ \mtrueconst} = \left\{ M =[M', M'', M_{11}, \dots, M_{kk}]  \textrm{ s.t. } \norm{M'} \leq \mtrueconst \textrm{ and } \norm{M''} \leq \mtrueconst \textrm{ and } \norm{M_{ij}} \leq \mtrueconst \right\}$.
\ENDFOR 
\end{algorithmic}
\end{algorithm}

Just as in Theorem~\ref{thm:lengthgeneralization} and Theorem~\ref{thm:lengthgeneralization_two_auto}, we have a companion result for tensorized spectral filtering. 
\begin{theorem}
\label{thm:lengthgeneralization_tensors}
Suppose $y_{1:t}$ evolves as a noiseless $(A,B,C,I)$-LDS where $A \in \reals^{d_A \times d_A}$ is a PSD matrix and the input $u_{T:1}$ is such that $ \sum_{t = 1}^T (T-t) u_t u_t^{\top} \succeq (2 \norm{C} \norm{B}/\sqrt{T})I$. Further, assume the eigenvalues of $A$ fall in the range $[0, 1 - \log(T)/(8T^q)] \cup [1-1/(2T^{1/4}), 1]$. Let $k  = \Omega \left(  \log(T) \cdot \log \left( T  d_A \right) \right)$, $r \geq \norm{B} \norm{C}$ and assume $T \geq (4 k  \log^2(T)/ \norm{C} \norm{B} )^{4}$. Algorithm~\ref{alg:tensor_sf} satisfies: 
$$     \unfairregret = \sum_{t=1}^T \ell_t( M^t, T^q) - \min_{M^* \in \K_{r}} \sum_{t=1}^T \ell_t(M^* , T) \leq  O \left( \norm{B}^2 \norm{C}^2 k^{3} \log^2(T) \sqrt{T} \right).$$
\end{theorem}

The proof structure of Theorem~\ref{thm:lengthgeneralization_tensors} is the same as as the proof structure used for Theorem~\ref{thm:lengthgeneralization} and Theorem~\ref{thm:lengthgeneralization_two_auto} as described in Section~\ref{subsection:high_level_proof}. The full proof of Theorem~\ref{thm:lengthgeneralization_tensors} is in Appendix~\ref{appendix:tensor_proof}.

\section{Experiments} \label{sec:experiments}
\subsection{Linear Dynamical System}
We can empirically verify Theorem~\ref{thm:lengthgeneralization} in an online sequence prediction task where the data is generated by a noiseless LDS. We refer to a ``bad" region of eigenvalues $\left(1 - {\log(T)}/({8T^{7/8}}),\; 1-{1}/({2T^{5/4}})\right)$ as Region B, and we define Region A to hug Region B on both sides as shown in Figure~\ref{fig:badregion}. 
\begin{figure}[ht]
    \centering
\includegraphics[width=0.8\textwidth]{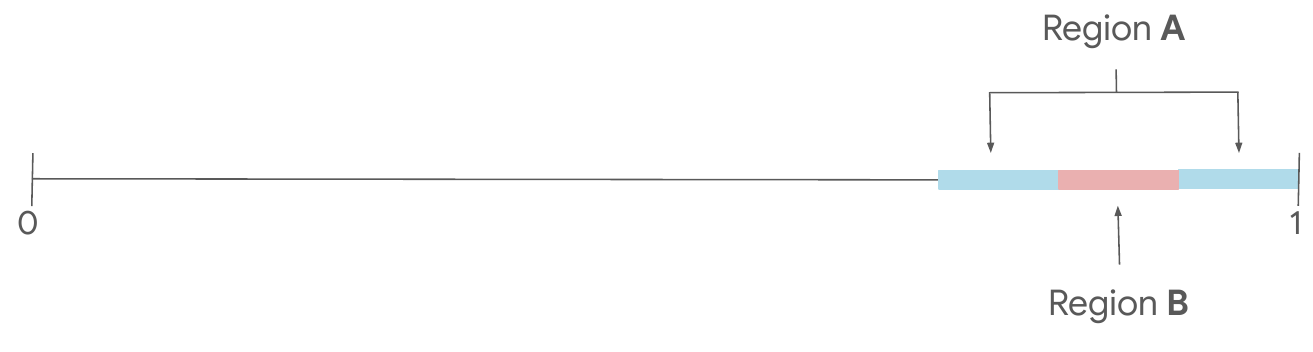}
    \caption{The red region (Region B) represents the interval of eigenvalues for which length generalization is not guaranteed by our main theorem. The blue region (Region A) is chosen to hug Region B on both sides -- to be precise, the leftmost point of Region A is $0.9 \cdot \left(1 - {\log(T)}/({8T^{7/8}})\right)$, and the rightmost point is $1$. This selection ensures that (1) Region $A$ will start to contain bad eigenvalues as $q$ decreases from $7/8$ and (2) eigenvalues in Region B are bad for $q \leq 7/8$.
    }
    \label{fig:badregion}
\end{figure}

Theorem~\ref{thm:lengthgeneralization} predicts that if all the eigenvalues lie outside Region B, then spectral filtering will length generalize from $T^{7/8}$ to $T$. To confirm this, we generate a random LDS of hidden dimension 512 with half of the LDS eigenvalues uniformly sampled from each component of \textbf{Region A}. The online prediction losses are plotted in Figure~\ref{fig:lds_huggingbad} for different choices of context length $T^q$, where $T=2^{14}$ and $k=24$. As expected from the theory, context lengths approaching $T^{7/8}$ closely match the performance of the optimal spectral filtering predictor with full context. 

Very interestingly, we see that context length $T^{1/2}$ consistently fails in a qualitatively worse fashion -- indeed, some of the values in Region A are actually ``bad" for $q=1/2$. This seems to suggest that such eigenvalues can actually cause instabilities/issues with length generalization and are not limitations of our proof -- if true, such a fact could be seen as a partial converse to Theorem~\ref{thm:lengthgeneralization} and would justify our use of ``bad'' to describe these eigenvalues. To check this conjecture empirically, we run another experiment where we generate a random LDS of hidden dimension 512 with all eigenvalues in \textbf{Region B} and plot the prediction losses in Figure \ref{fig:lds_bad}. These results confirm that (some subset of) this bad region is indeed what throws off the length generalization capability of spectral filtering\footnote{Note that these plots are zoomed in to inspect the convergence to a loss of exactly 0; in practice, loss values below $0.1$ are still meaningfully small.}.

\begin{figure}[ht]
    \centering
    \begin{minipage}{0.45\textwidth}
    \centering
\includegraphics[width=1.0\linewidth]{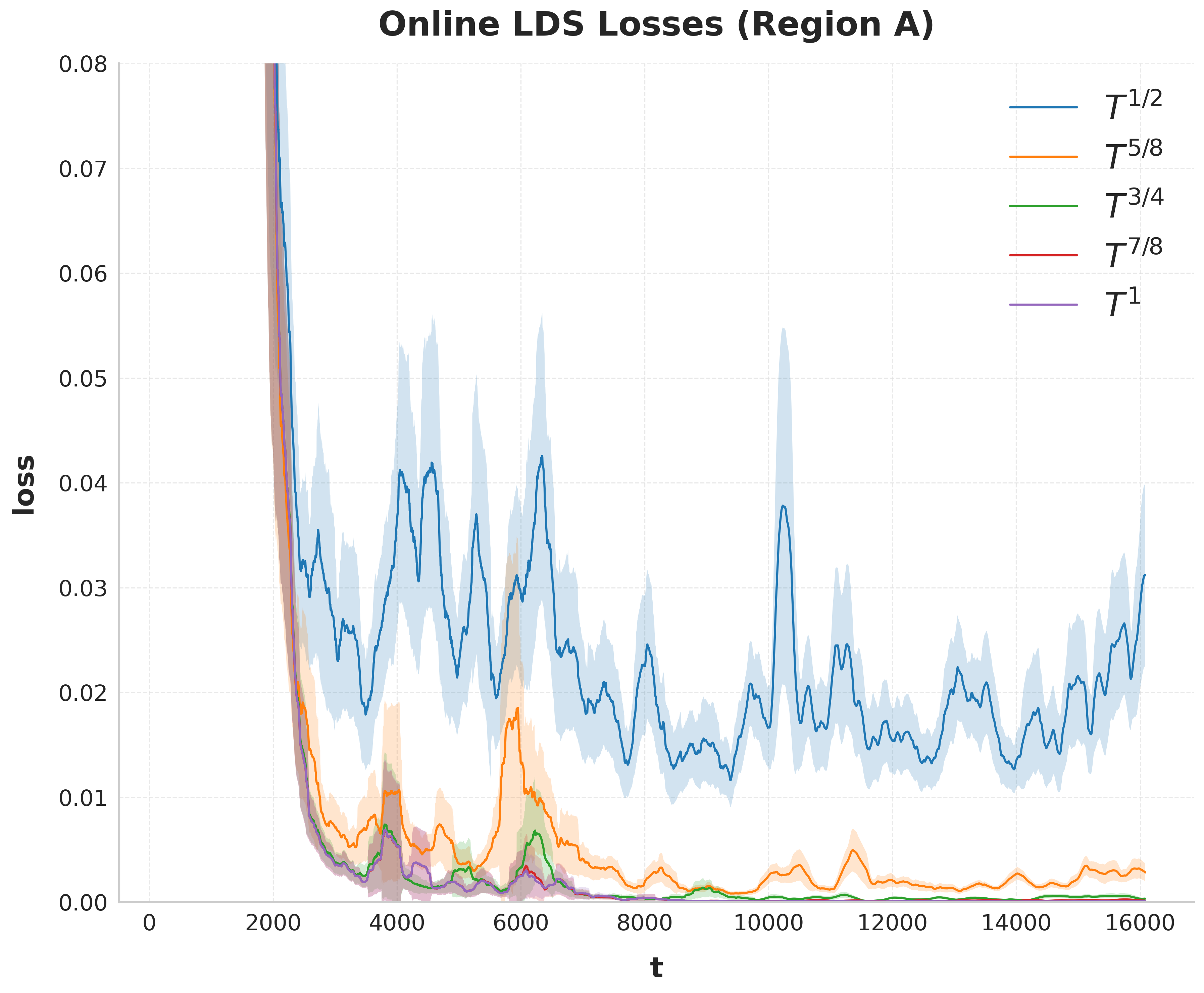}
    \caption{Prediction losses $\ell_t(M^t, T^q)$ as a function of $t$ on an LDS with eigenvalues sampled from \textbf{Region A}, averaged over random seeds and smoothed.
    }
    \label{fig:lds_huggingbad}
    \end{minipage}
    \hfill
    \begin{minipage}{0.45\textwidth}
        \centering
        \includegraphics[width=1.0\linewidth]{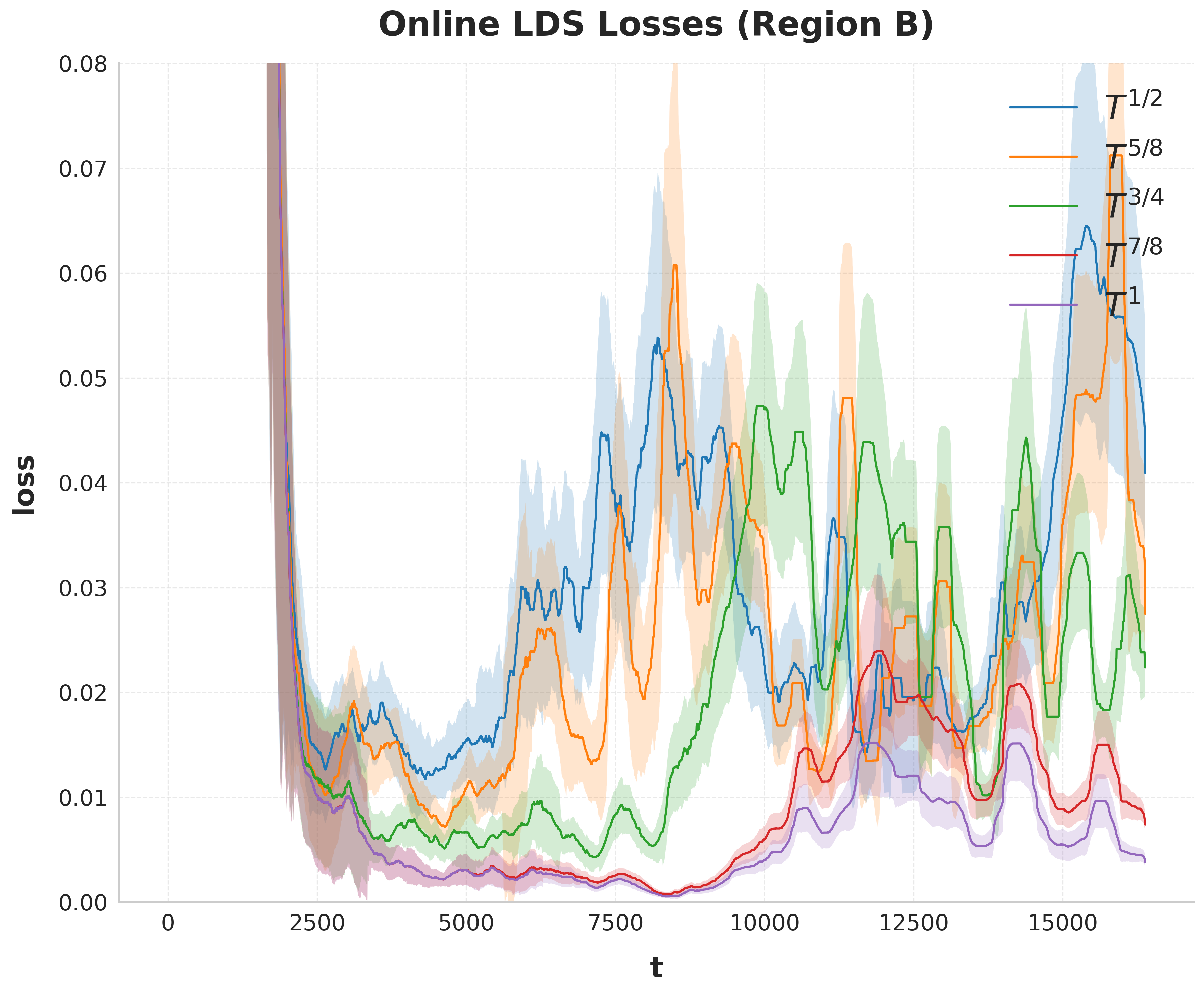}
    \caption{Prediction losses $\ell_t(M^t, T^q)$ as a function of $t$ on an LDS with eigenvalues sampled from \textbf{Region B}, averaged over random seeds and smoothed.}
    \label{fig:lds_bad}
    \end{minipage}
\end{figure}

\subsection{Two Autoregressive Components}
We have seen, both in theory and experiment, that vanilla spectral filtering has an inherent length generalization capability on an LDS under a minor spectral assumption. Introducing a second autoregressive component yields Algorithm \ref{alg:sf_two}, which is accompanied by a length generalization guarantee that removes this assumption and applies to all (symmetric, marginally-stable) LDS's. We verify this experimentally in Figure \ref{fig:lds_twoautoregressive} -- to be as adversarial as we can, this experiment is run with all eigenvalues sampled from \textbf{Region B}. As predicted by Theorem \ref{thm:lengthgeneralization_two_auto}, the second autoregressive component allows for robust length generalization even with context lengths as small as $\sqrt{T}$\footnote{The exact same conclusions are true for the tensorized version of the spectral filtering algorithm: the "bad" eigenvalues cause problems if only $y_{t-1}$ is used, and introducing $y_{t-2}$ allows for complete length generalization under no spectral assumptions. The regression on $u_{t-1}$ and $u_{t-2}$ appears to make no difference in practice. 

}.

\begin{figure}[ht]
    \centering
    \includegraphics[width=0.45\linewidth]{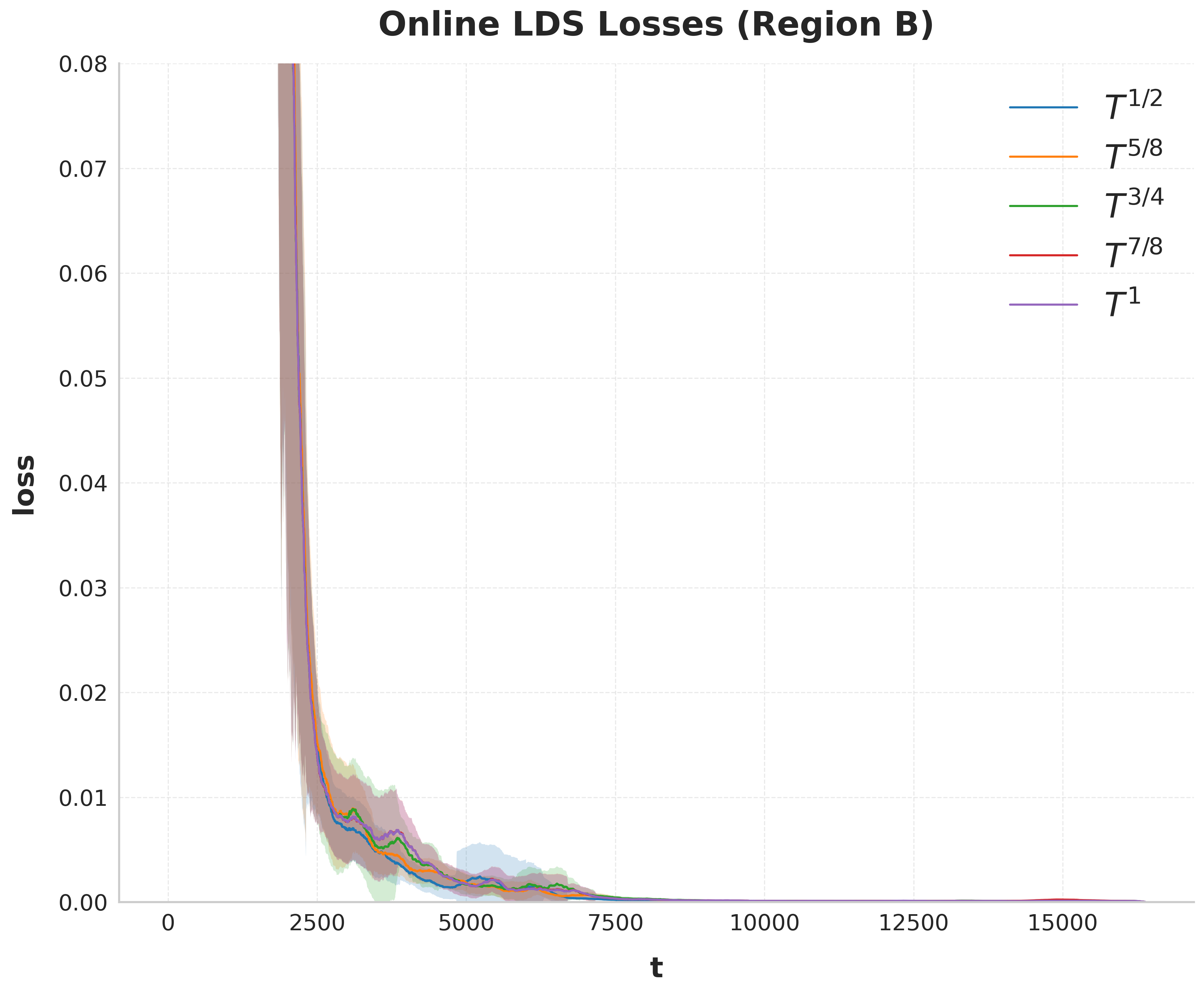}
    \caption{Prediction losses $\ell_t(M^t, T^q)$ as a function of $t$ with \textbf{two autoregressive components} on an LDS with eigenvalues sampled from \textbf{Region B}, averaged over random seeds and smoothed. Contrast with Figure \ref{fig:lds_bad}.}
    \label{fig:lds_twoautoregressive}
\end{figure}

\subsection{Induction Heads}
So far, we have demonstrated length generalization of spectral filtering on linear systems: when trained with a shorter context length of $T^q$ it is able to compete with methods that have access to the full context $T$ (even on marginally-stable systems that can have arbitrarily large effective memory lengths). This length generalization property is most crucial in deep learning applications, in which multi-layer models are stacked (with added nonlinearities) to solve non-LDS sequence prediction tasks.

As an empirical proof-of-concept to demonstrate that STU's length generalization capability extends to this regime, we evaluate it on the induction heads synthetic sequence modeling task, which is commonplace in the language modeling literature (see \cite{gu2023mamba})
and was experimentally shown in \cite{liu2024flash} to be efficiently solved by a two-layer STU. In the induction heads task, the model is required to recall one token (sampled uniformly from a vocabulary) immediately after a special \verb|flag| token; the rest of the sequence consists of the same special \verb|blank| token, which the model should learn to ignore. 

The STU architecture we use is composed of an embedding layer, two "tensordot" STU layers with MLPs and ReLU nonlinearities, and an output projection layer (the same as in \cite{liu2024flash}) with filters of length $T = 256$.

Following prior STU architecture implementations we use \textbf{no autoregressive components}, and so any length generalization observed here comes directly from the filtering mechanism itself. We train these models until convergence with a tuned Adam optimizer and various context lengths $T^q$. The vocabulary size is set to 4.

\begin{figure}[ht]
    \centering
    \includegraphics[width=0.65\linewidth]{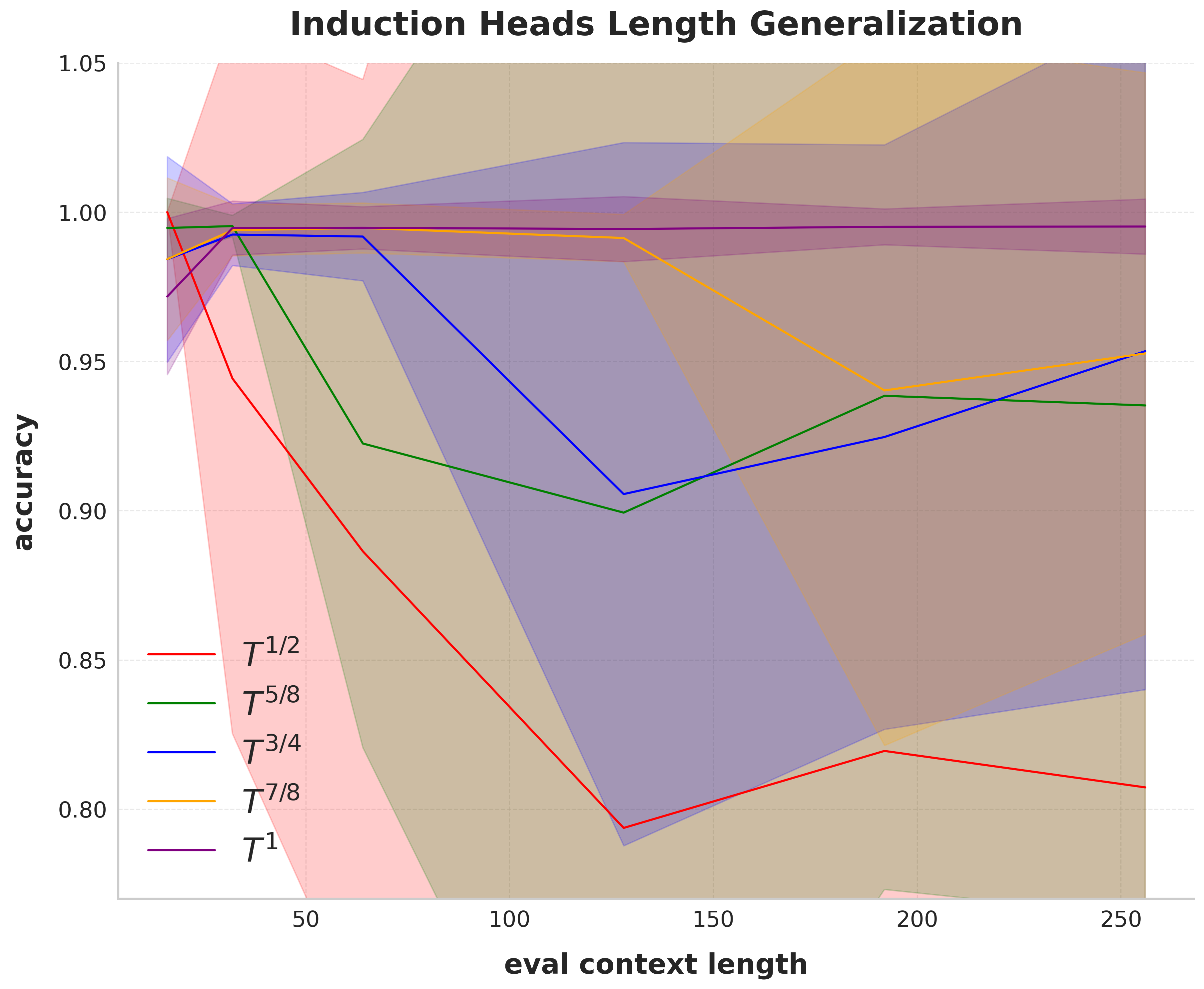}
    \caption{Accuracies for STU models trained on an induction heads task of length $T^q$ and evaluated on sequence lengths increasing up to $T$, averaged over random seeds. Models occasionally generalize all the way up to length $T$, as indicated by the large variance of evaluation accuracies. }
    \label{fig:inductionheads_anniesweep}
\end{figure}

Accuracies are plotted in Figure~\ref{fig:inductionheads_anniesweep} for evaluation task lengths increasing up to $T$. As we see, vanilla STU models are able to nontrivially length generalize and occasionally retain good accuracy beyond their training context lengths, though very inconsistently. Importantly, unlike algorithms that achieve length generalization through architectural modification, we simply just train with filters longer than the train context\footnote{For example, the nonlinear selection mechanism of \cite{gu2023mamba} allows for extreme length generalization on induction heads without prior knowledge of the evaluation length, though at a cost to training efficiency, implementation simplicity, and optimization complexity. We reiterate that our goal is not to navigate such a tradeoff by modifying the STU model so that it length generalizes on induction heads, but rather to exhibit a provable length generalization capability of the STU that comes for free from its natural structure.}. As such, this method allows for the convolutional mode during training and inherits all the benefits of STU that are demonstrated in \cite{liu2024flash}.

\subsection{Tensorized STU Models}
On the LDS, tensorization made no difference in terms of length generalization capabilities. Intuitively, however, the tensor representation does have a natural structure that seems conducive to length generalization. The same machinery (the first tensor component) is shared when learning the system's evolution over each chunk of size $\sqrt{T}$, and the second tensor component simply aggregates the responses from these chunks. Crucially, the tensorized model is more expressive than an LDS since it does not require the tensor components to learn the pair $(\alpha, \alpha^{\sqrt{T}})$ but instead allows any $(\alpha, \beta)$ pair. In other words, the tensorized model is free to decouple how it gathers information within the chunks from how it synthesizes information across chunks. This extra expressitivity is unnecessary to succeed on an LDS, but turns out to matter on more complex nonlinear tasks, which we now investigate.

We repeat the experimental setup of the previous subsection, with the architectural difference that we replace the $k$ many filters of length $256$ with $k^2$ many filters of length $256$ formed from tensor combinations with components of length $16$. Note that although we increase the size of our filter bank from $k$ to $k^2$, this does not change the number of convolutions due to the tensordot approximation (see \cite{liu2024flash} for details). 

In addition to induction heads, we also experiment on a copy task \cite{arjovsky2016unitary}. This is a nonlinear sequence-to-sequence prediction task in which a fixed number of tokens are randomly sampled from a vocabulary and placed at the beginning of the input sequence, with \verb|blank| tokens filling the rest of the sequence -- the model must recall (in order) the initial non-blank tokens. We use a vocabulary size of 4 and fill $\frac{2}{3}$ of each input sequence with \verb|blank| tokens.

In Figures \ref{fig:inductionheads_lengthgen} and \ref{fig:copy_lengthgenn} we plot the length generalization results on induction heads and copy, respectively, for both the vanilla STU and the variant with tensorized filters. We see that the tensorized model improves in both cases, with greater benefit at more extreme length generalizations (i.e. at smaller $q$). On the copy task in particular, it seems tensorization allows for near-perfect length generalization which the vanilla STU model cannot accomplish. This makes sense given the nature of the task: if we imagine splitting the filters into chunks of size $\sqrt{T}$, the issue is that tokens to be copied fall on different chunks during training and evaluation. The shared computation and extra expressivity of the tensorized model addresses this gap perfectly, allowing for the solution that is learned on a short instance of the task to be deployed on larger ones. By contrast, even though the vanilla STU model is able to "see" the whole input during evaluation (because it has filter length 256), it is unable to length generalize (or even learn easily at long lengths given full context, as the dotted $T^1$ curve shows).

\begin{figure}[ht]
    \centering
    \begin{minipage}{0.45\textwidth}
    \centering
\includegraphics[width=1.0\linewidth]{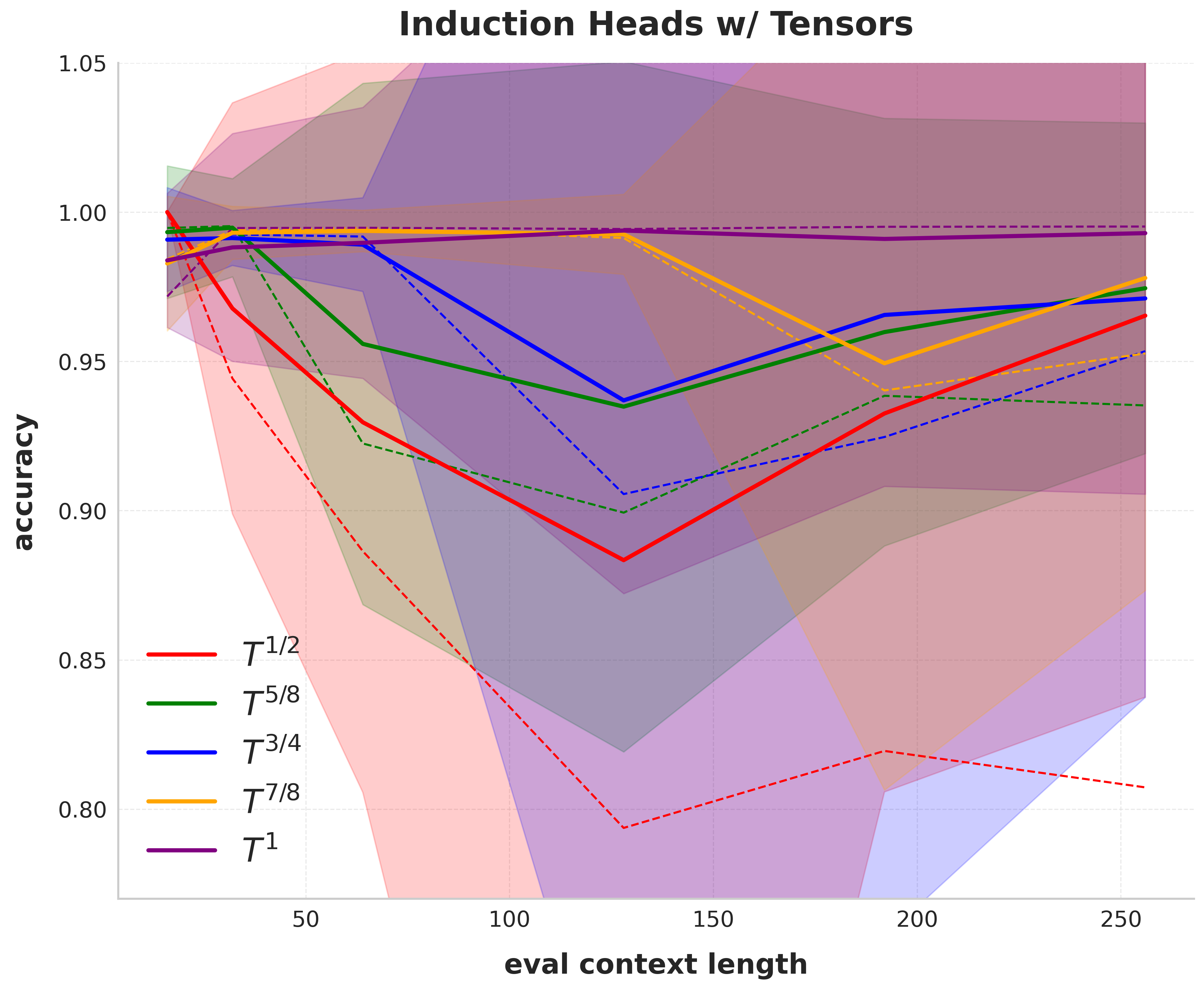}
    \caption{Length generalization on induction heads. Experiments with \textbf{tensorized filters} are bolded, while those without are dotted. Averaged over random seeds. 
    }
    \label{fig:inductionheads_lengthgen}
    \end{minipage}
    \hfill
    \begin{minipage}{0.45\textwidth}
        \centering
        \includegraphics[width=1.0\linewidth]{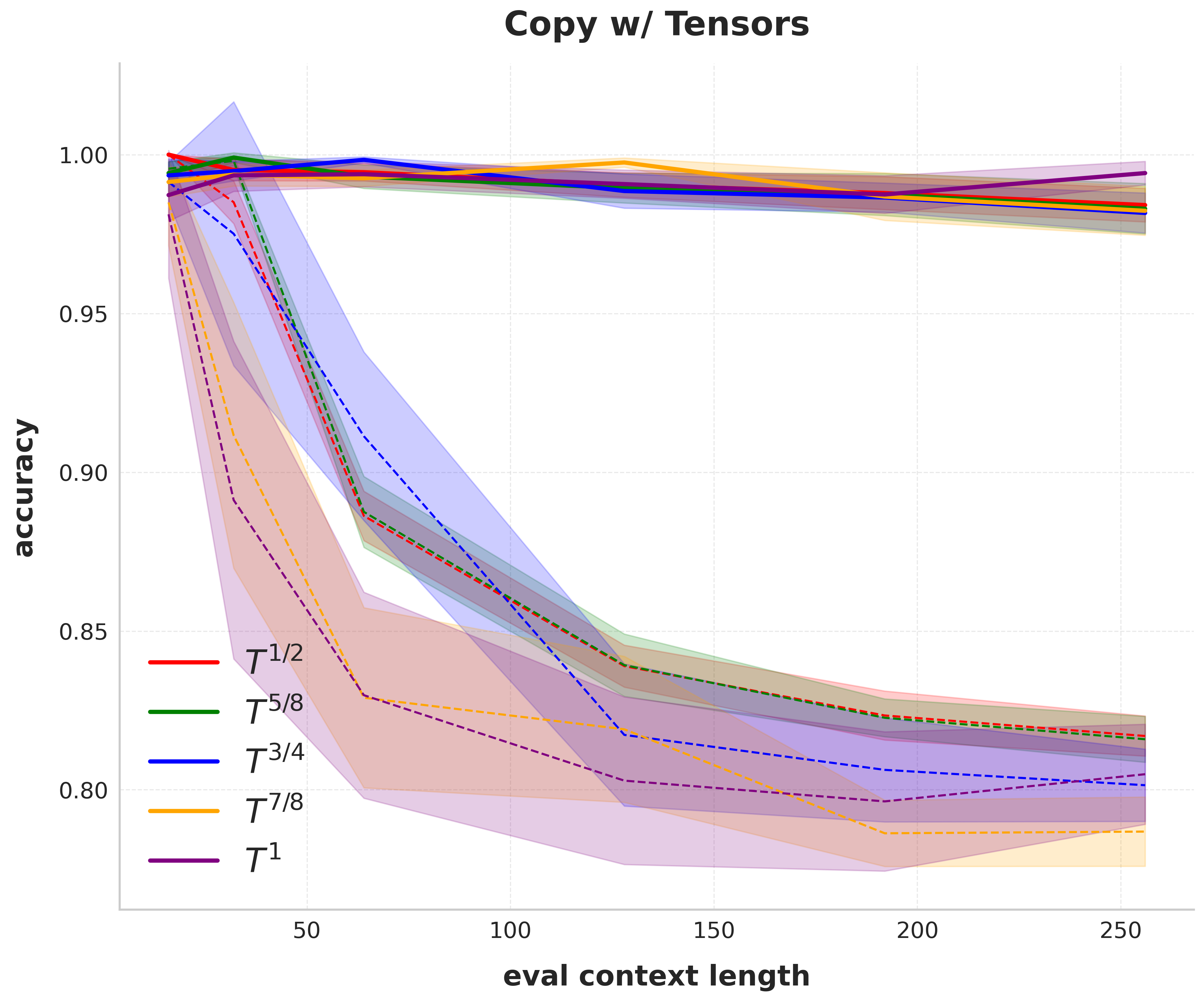}
    \caption{Length generalization on the copy task. Experiments with \textbf{tensorized filters} are bolded, while those without are dotted. Averaged over random seeds.}
    \label{fig:copy_lengthgenn}
    \end{minipage}
\end{figure}

As before, we remark that tensorized filters are not a substantial change to the regular STU architecture, but only a different set of filters to convolve against: such a model still retains favorable training efficiency via convolutions, optimization simplicity, etc.. This empirical result of improved length generalization via tensorization is not yet explained by theory, but in a sense it also seems to "comes for free" with this construction of tensorized spectral filters. We conjecture the improved performance comes from the more expressive architecture, since standard filters are strictly contained in tensorized filters.  We leave a more rigorous explanation of this phenomenon and experiments at larger scales for future work.

\section{Discussion}
In review, we first introduced the notion of \unfairregrettext\ as a way to describe length generalization through the lens of online learning and regret minimization in games. We then proved that the class of spectral filtering predictors naturally enjoys sublinear \unfairregrettext\, thereby exhibiting length generalization without any change to the algorithm. Next, we used experiments on synthetic data generated by an LDS to demonstrate the validity and sharpness of our theory and provided proof-of-concept length generalization experiments on a synthetic nonlinear sequence prediction task. Finally, we made use of the tensor structure inherent in LDS's to design a more expressive tensorized spectral filtering algorithm, which empirically boosted length generalization on our synthetic nonlinear tasks.

Our theoretical results and initial empirical findings reveal that some type of length generalization comes naturally with the spectral filtering algorithm. This adds to the already-exciting list of its useful (and provable) properties, including: robustness to systems with long memory and large hidden dimension, efficient training via convolutions, optimization convexity, and the existence of very parameter-efficient approximations. Given recent successful applications of spectral filtering as the building block for STU models in deep learning \citep{agarwal2023spectral,liu2024flash}, it would be valuable to research how to best take advantage of their length generalization capacity at scale -- we leave this for future work.
\newpage

\bibliographystyle{iclr2025_conference}
\bibliography{main}

\begin{thebibliography}{36}
\providecommand{\natexlab}[1]{#1}
\providecommand{\url}[1]{\texttt{#1}}
\expandafter\ifx\csname urlstyle\endcsname\relax
  \providecommand{\doi}[1]{doi: #1}\else
  \providecommand{\doi}{doi: \begingroup \urlstyle{rm}\Url}\fi

\bibitem[Abbe et~al.(2023)Abbe, Bengio, Lotfi, and Rizk]{abbe2023generalization}
Emmanuel Abbe, Samy Bengio, Aryo Lotfi, and Kevin Rizk.
\newblock Generalization on the unseen, logic reasoning and degree curriculum.
\newblock In \emph{International Conference on Machine Learning}, pp.\  31--60. PMLR, 2023.

\bibitem[Agarwal et~al.(2023)Agarwal, Suo, Chen, and Hazan]{agarwal2023spectral}
Naman Agarwal, Daniel Suo, Xinyi Chen, and Elad Hazan.
\newblock Spectral state space models.
\newblock \emph{arXiv preprint arXiv:2312.06837}, 2023.

\bibitem[Anil et~al.(2022)Anil, Wu, Andreassen, Lewkowycz, Misra, Ramasesh, Slone, Gur-Ari, Dyer, and Neyshabur]{anil2022exploring}
Cem Anil, Yuhuai Wu, Anders Andreassen, Aitor Lewkowycz, Vedant Misra, Vinay Ramasesh, Ambrose Slone, Guy Gur-Ari, Ethan Dyer, and Behnam Neyshabur.
\newblock Exploring length generalization in large language models.
\newblock \emph{Advances in Neural Information Processing Systems}, 35:\penalty0 38546--38556, 2022.

\bibitem[Arjovsky et~al.(2016)Arjovsky, Shah, and Bengio]{arjovsky2016unitary}
Martin Arjovsky, Amar Shah, and Yoshua Bengio.
\newblock Unitary evolution recurrent neural networks.
\newblock In \emph{International conference on machine learning}, pp.\  1120--1128. PMLR, 2016.

\bibitem[Brown et~al.(2020)Brown, Mann, Ryder, Subbiah, Kaplan, Dhariwal, Neelakantan, Shyam, Sastry, Askell, et~al.]{brown2020language}
Tom Brown, Benjamin Mann, Nick Ryder, Melanie Subbiah, Jared~D Kaplan, Prafulla Dhariwal, Arvind Neelakantan, Pranav Shyam, Girish Sastry, Amanda Askell, et~al.
\newblock Language models are few-shot learners.
\newblock \emph{Advances in neural information processing systems}, 33:\penalty0 1877--1901, 2020.

\bibitem[Cesa-Bianchi et~al.(2004)Cesa-Bianchi, Conconi, and Gentile]{cesa2004generalization}
Nicolo Cesa-Bianchi, Alex Conconi, and Claudio Gentile.
\newblock On the generalization ability of on-line learning algorithms.
\newblock \emph{IEEE Transactions on Information Theory}, 50\penalty0 (9):\penalty0 2050--2057, 2004.

\bibitem[Chi et~al.(2022)Chi, Fan, Ramadge, and Rudnicky]{chi2022kerple}
Ta-Chung Chi, Ting-Han Fan, Peter~J Ramadge, and Alexander Rudnicky.
\newblock Kerple: Kernelized relative positional embedding for length extrapolation.
\newblock \emph{Advances in Neural Information Processing Systems}, 35:\penalty0 8386--8399, 2022.

\bibitem[Dai(2019)]{dai2019transformer}
Zihang Dai.
\newblock Transformer-xl: Attentive language models beyond a fixed-length context.
\newblock \emph{arXiv preprint arXiv:1901.02860}, 2019.

\bibitem[Del{\'e}tang et~al.(2022)Del{\'e}tang, Ruoss, Grau-Moya, Genewein, Wenliang, Catt, Cundy, Hutter, Legg, Veness, et~al.]{deletang2022neural}
Gr{\'e}goire Del{\'e}tang, Anian Ruoss, Jordi Grau-Moya, Tim Genewein, Li~Kevin Wenliang, Elliot Catt, Chris Cundy, Marcus Hutter, Shane Legg, Joel Veness, et~al.
\newblock Neural networks and the chomsky hierarchy.
\newblock \emph{arXiv preprint arXiv:2207.02098}, 2022.

\bibitem[Dosovitskiy et~al.(2020)Dosovitskiy, Beyer, Kolesnikov, Weissenborn, Zhai, Unterthiner, Dehghani, Minderer, Heigold, Gelly, et~al.]{dosovitskiy2020image}
Alexey Dosovitskiy, Lucas Beyer, Alexander Kolesnikov, Dirk Weissenborn, Xiaohua Zhai, Thomas Unterthiner, Mostafa Dehghani, Matthias Minderer, Georg Heigold, Sylvain Gelly, et~al.
\newblock An image is worth 16x16 words: Transformers for image recognition at scale.
\newblock \emph{arXiv preprint arXiv:2010.11929}, 2020.

\bibitem[Dziri et~al.(2024)Dziri, Lu, Sclar, Li, Jiang, Lin, Welleck, West, Bhagavatula, Le~Bras, et~al.]{dziri2024faith}
Nouha Dziri, Ximing Lu, Melanie Sclar, Xiang~Lorraine Li, Liwei Jiang, Bill~Yuchen Lin, Sean Welleck, Peter West, Chandra Bhagavatula, Ronan Le~Bras, et~al.
\newblock Faith and fate: Limits of transformers on compositionality.
\newblock \emph{Advances in Neural Information Processing Systems}, 36, 2024.

\bibitem[Gu \& Dao(2023)Gu and Dao]{gu2023mamba}
Albert Gu and Tri Dao.
\newblock Mamba: Linear-time sequence modeling with selective state spaces.
\newblock \emph{arXiv preprint arXiv:2312.00752}, 2023.

\bibitem[Gu et~al.(2020)Gu, Dao, Ermon, Rudra, and R\'{e}]{NEURIPS2020hippo}
Albert Gu, Tri Dao, Stefano Ermon, Atri Rudra, and Christopher R\'{e}.
\newblock Hippo: Recurrent memory with optimal polynomial projections.
\newblock In H.~Larochelle, M.~Ranzato, R.~Hadsell, M.F. Balcan, and H.~Lin (eds.), \emph{Advances in Neural Information Processing Systems}, volume~33, pp.\  1474--1487. Curran Associates, Inc., 2020.

\bibitem[Gu et~al.(2021{\natexlab{a}})Gu, Goel, and R{\'e}]{gu2021efficiently}
Albert Gu, Karan Goel, and Christopher R{\'e}.
\newblock Efficiently modeling long sequences with structured state spaces.
\newblock \emph{arXiv preprint arXiv:2111.00396}, 2021{\natexlab{a}}.

\bibitem[Gu et~al.(2021{\natexlab{b}})Gu, Johnson, Goel, Saab, Dao, Rudra, and R{\'e}]{gu2021combining}
Albert Gu, Isys Johnson, Karan Goel, Khaled Saab, Tri Dao, Atri Rudra, and Christopher R{\'e}.
\newblock Combining recurrent, convolutional, and continuous-time models with linear state space layers.
\newblock \emph{Advances in neural information processing systems}, 34:\penalty0 572--585, 2021{\natexlab{b}}.

\bibitem[Gupta et~al.(2022)Gupta, Gu, and Berant]{gupta2022diagonal}
Ankit Gupta, Albert Gu, and Jonathan Berant.
\newblock Diagonal state spaces are as effective as structured state spaces.
\newblock In Alice~H. Oh, Alekh Agarwal, Danielle Belgrave, and Kyunghyun Cho (eds.), \emph{Advances in Neural Information Processing Systems}, 2022.
\newblock URL \url{https://openreview.net/forum?id=RjS0j6tsSrf}.

\bibitem[Hazan \& Singh(2022)Hazan and Singh]{hazan2022introduction}
Elad Hazan and Karan Singh.
\newblock Introduction to online nonstochastic control.
\newblock \emph{arXiv preprint arXiv:2211.09619}, 2022.

\bibitem[Hazan et~al.(2017{\natexlab{a}})Hazan, Singh, and Zhang]{hazan2017efficient}
Elad Hazan, Karan Singh, and Cyril Zhang.
\newblock Efficient regret minimization in non-convex games.
\newblock In \emph{International Conference on Machine Learning}, pp.\  1433--1441. PMLR, 2017{\natexlab{a}}.

\bibitem[Hazan et~al.(2017{\natexlab{b}})Hazan, Singh, and Zhang]{hazan2017learning}
Elad Hazan, Karan Singh, and Cyril Zhang.
\newblock Learning linear dynamical systems via spectral filtering.
\newblock \emph{Advances in Neural Information Processing Systems}, 30, 2017{\natexlab{b}}.

\bibitem[Hazan et~al.(2018)Hazan, Lee, Singh, Zhang, and Zhang]{hazan2018spectral}
Elad Hazan, Holden Lee, Karan Singh, Cyril Zhang, and Yi~Zhang.
\newblock Spectral filtering for general linear dynamical systems.
\newblock \emph{Advances in Neural Information Processing Systems}, 31, 2018.

\bibitem[Hazan et~al.(2016)]{hazan2016introduction}
Elad Hazan et~al.
\newblock Introduction to online convex optimization.
\newblock \emph{Foundations and Trends{\textregistered} in Optimization}, 2\penalty0 (3-4):\penalty0 157--325, 2016.

\bibitem[Hou et~al.(2024)Hou, Brandfonbrener, Kakade, Jelassi, and Malach]{hou2024universal}
Kaiying Hou, David Brandfonbrener, Sham Kakade, Samy Jelassi, and Eran Malach.
\newblock Universal length generalization with turing programs.
\newblock \emph{arXiv preprint arXiv:2407.03310}, 2024.

\bibitem[Jelassi et~al.(2023)Jelassi, d'Ascoli, Domingo-Enrich, Wu, Li, and Charton]{jelassi2023length}
Samy Jelassi, St{\'e}phane d'Ascoli, Carles Domingo-Enrich, Yuhuai Wu, Yuanzhi Li, and Fran{\c{c}}ois Charton.
\newblock Length generalization in arithmetic transformers.
\newblock \emph{arXiv preprint arXiv:2306.15400}, 2023.

\bibitem[Jumper et~al.(2021)Jumper, Evans, Pritzel, Green, Figurnov, Ronneberger, Tunyasuvunakool, Bates, {\v{Z}}{\'\i}dek, Potapenko, et~al.]{jumper2021highly}
John Jumper, Richard Evans, Alexander Pritzel, Tim Green, Michael Figurnov, Olaf Ronneberger, Kathryn Tunyasuvunakool, Russ Bates, Augustin {\v{Z}}{\'\i}dek, Anna Potapenko, et~al.
\newblock Highly accurate protein structure prediction with alphafold.
\newblock \emph{Nature}, 596\penalty0 (7873):\penalty0 583--589, 2021.

\bibitem[Kazemnejad et~al.(2024)Kazemnejad, Padhi, Natesan~Ramamurthy, Das, and Reddy]{kazemnejad2024impact}
Amirhossein Kazemnejad, Inkit Padhi, Karthikeyan Natesan~Ramamurthy, Payel Das, and Siva Reddy.
\newblock The impact of positional encoding on length generalization in transformers.
\newblock \emph{Advances in Neural Information Processing Systems}, 36, 2024.

\bibitem[Li et~al.(2023)Li, You, Guruganesh, Ainslie, Ontanon, Zaheer, Sanghai, Yang, Kumar, and Bhojanapalli]{li2023functional}
Shanda Li, Chong You, Guru Guruganesh, Joshua Ainslie, Santiago Ontanon, Manzil Zaheer, Sumit Sanghai, Yiming Yang, Sanjiv Kumar, and Srinadh Bhojanapalli.
\newblock Functional interpolation for relative positions improves long context transformers.
\newblock \emph{arXiv preprint arXiv:2310.04418}, 2023.

\bibitem[Liu et~al.(2024)Liu, Nguyen, Devre, Dogariu, Majumdar, and Hazan]{liu2024flash}
Y~Isabel Liu, Windsor Nguyen, Yagiz Devre, Evan Dogariu, Anirudha Majumdar, and Elad Hazan.
\newblock Flash stu: Fast spectral transform units.
\newblock \emph{arXiv preprint arXiv:2409.10489}, 2024.

\bibitem[Orvieto et~al.(2023)Orvieto, Smith, Gu, Fernando, Gulcehre, Pascanu, and De]{orvieto2023resurrecting}
Antonio Orvieto, Samuel~L Smith, Albert Gu, Anushan Fernando, Caglar Gulcehre, Razvan Pascanu, and Soham De.
\newblock Resurrecting recurrent neural networks for long sequences.
\newblock \emph{arXiv preprint arXiv:2303.06349}, 2023.

\bibitem[Press et~al.(2021)Press, Smith, and Lewis]{press2021train}
Ofir Press, Noah~A Smith, and Mike Lewis.
\newblock Train short, test long: Attention with linear biases enables input length extrapolation.
\newblock \emph{arXiv preprint arXiv:2108.12409}, 2021.

\bibitem[Press et~al.(2022)Press, Smith, and Lewis]{press2022trainshorttestlong}
Ofir Press, Noah~A. Smith, and Mike Lewis.
\newblock Train short, test long: Attention with linear biases enables input length extrapolation, 2022.
\newblock URL \url{https://arxiv.org/abs/2108.12409}.

\bibitem[Shen et~al.(2023)Shen, Bubeck, Eldan, Lee, Li, and Zhang]{shen2023positional}
Ruoqi Shen, S{\'e}bastien Bubeck, Ronen Eldan, Yin~Tat Lee, Yuanzhi Li, and Yi~Zhang.
\newblock Positional description matters for transformers arithmetic.
\newblock \emph{arXiv preprint arXiv:2311.14737}, 2023.

\bibitem[Smith et~al.(2023)Smith, Warrington, and Linderman]{smith2023simplified}
Jimmy~T.H. Smith, Andrew Warrington, and Scott Linderman.
\newblock Simplified state space layers for sequence modeling.
\newblock In \emph{The Eleventh International Conference on Learning Representations}, 2023.

\bibitem[Tay et~al.(2022)Tay, Dehghani, Bahri, and Metzler]{10.1145/3530811}
Yi~Tay, Mostafa Dehghani, Dara Bahri, and Donald Metzler.
\newblock Efficient transformers: A survey.
\newblock \emph{ACM Comput. Surv.}, 55\penalty0 (6), dec 2022.
\newblock ISSN 0360-0300.
\newblock \doi{10.1145/3530811}.
\newblock URL \url{https://doi.org/10.1145/3530811}.

\bibitem[Vaswani et~al.(2017)Vaswani, Shazeer, Parmar, Uszkoreit, Jones, Gomez, Kaiser, and Polosukhin]{vaswani2017attention}
Ashish Vaswani, Noam Shazeer, Niki Parmar, Jakob Uszkoreit, Llion Jones, Aidan~N Gomez, {\L}ukasz Kaiser, and Illia Polosukhin.
\newblock Attention is all you need.
\newblock \emph{Advances in neural information processing systems}, 30, 2017.

\bibitem[Zhou et~al.(2023)Zhou, Bradley, Littwin, Razin, Saremi, Susskind, Bengio, and Nakkiran]{zhou2023algorithms}
Hattie Zhou, Arwen Bradley, Etai Littwin, Noam Razin, Omid Saremi, Josh Susskind, Samy Bengio, and Preetum Nakkiran.
\newblock What algorithms can transformers learn? a study in length generalization.
\newblock \emph{arXiv preprint arXiv:2310.16028}, 2023.

\bibitem[Zhou et~al.(2024)Zhou, Alon, Chen, Wang, Agarwal, and Zhou]{zhou2024transformers}
Yongchao Zhou, Uri Alon, Xinyun Chen, Xuezhi Wang, Rishabh Agarwal, and Denny Zhou.
\newblock Transformers can achieve length generalization but not robustly.
\newblock \emph{arXiv preprint arXiv:2402.09371}, 2024.

\end{thebibliography}

\newpage
\appendix
\section{General Length Generalization}
\label{appendix:general}
In this section we introduce a general algorithm which we will use to prove length generalization for our specific Algorithms Algorithm~\ref{alg:ogd_short_length}, Algorithm~\ref{alg:sf_two}, and Algorithm~\ref{alg:tensor_sf}. 
\begin{algorithm}[h]
\caption{General Spectral Filtering} \label{alg:general_sf}
\begin{algorithmic}[1]
\STATE {\bf Input:} $k > 0, L >0$, $r > 0$, functions $p_t(\cdot)$, vectors $v_{1:k}$. Initialize $M_i=0$ for $i \in [k]$. 
\FOR {$t = 1,2,...,T$}
\STATE Compute and predict  
\begin{equation*} 
\hat{y}_t = p_t(y_{t-1:1})  +   \sum_{i=1}^{k} M_{i}  u_{t-1:t-L}v_i.
\end{equation*}
\STATE Observe $y_t$, denote $\ell_t(M^t, L) = \|\hat{y_t} - y_t\|^2$ and update and project update and project onto the low Frobenius norm ball
$$ \hat{M}^{t+1} \leftarrow M^{t} - \eta_t \nabla_{M} \ell_t( M^t ) $$
$$ M_{t+1} = \proj_{\K} \left( \hat{M}_{t+1} \right), $$
where $\K_{\mtrueconst} = \left\{ M  \textrm{ s.t. } \norm{M_i} \leq \mtrueconst  \right\}$.
\ENDFOR 
\end{algorithmic}
\end{algorithm}

Our workhorse theorem is presented below. We will use this theorem to prove length generalization for our special cases in the following sections. 
\begin{theorem}
\label{thm:general_length_gen}
     Suppose $y_{1:t}$ evolves as a noiseless $(A,B,C,I)$-LDS and the input $u_{(t-1):0}$ is such that $ \sum_{t = 0}^{T-1} (T-t) u_t u_t^{\top} \succeq (2 \norm{C} \norm{B}/\sqrt{T})I$. Let $k$, $L$, $r$, $\{ v_i \}_{i = 1}^k$, $p_t(\cdot)$, and $\ell_t(\cdot)$ all be as defined in Algorithm~\ref{alg:general_sf}. Suppose $\{ v_i \}_{i = 1}^k$ is orthonormal with $\norm{v_i}_1 \leq \log^p(T)$. Suppose that $p_t(\cdot)$ is such that there exists some function $h(\cdot)$, constant $\ell > 0$, and some $\mtrue \in \K_{\mtrueconst}$ such that
\begin{equation*}
    y_t - p_t(y_{t-1:1})  = \sum_{i = 1}^T \mtrue_i u_{t-1:0} v_i = \sum_{i = 1}^{\ell_1} \mtrue_i u_{t-i} + \sum_{i = 1}^{t-\ell_1} C A^i h(A) B u_{t - \ell_1 - i}, 
\end{equation*}
where
\begin{equation*}
    \norm{\sum_{i = k +1}^T \mtrue_{i} u_{t-1:t-L} v_i} \leq \norm{C} \norm{B}/T,
\end{equation*}
and
\begin{equation*}
    \max_{\alpha(A)} \left \{ h(\alpha) \alpha^{L- \ell_1 - 1} (1 - \alpha^{T - L + 1}) (1 -\alpha)^{-1} \right \} \leq \frac{1}{T^{1/4}}.
\end{equation*} 
Then if $M^t$ are the iterates of Algorithm~\ref{alg:general_sf} and $T \geq (4 k  \log^p(T)/ \norm{C} \norm{B} )^{4}$,
\begin{equation*}
    \sum_{t = 1}^T \ell_t(M^t, L) - \min_{M^* \in \K_{r}} \sum_{t=1}^T \ell_t(M^* , T) \leq \left( 12 k^{3/2} \mtrueconst^2 \log^p(T) + 8 \norm{C}^2 \norm{B}^2 \right) \sqrt{T}.
\end{equation*}
\end{theorem}
The proof of this theorem requires several technical lemmas which we present and prove in the subsequent subsections. In Lemma~\ref{lemma:ogd_regret_general} we essentially prove the standard result showing that Online Gradient Descent implemented in Algorithm~\ref{alg:general_sf} achieves $O(\sqrt{T})$ regret. In Lemma~\ref{lemma:length_generalization_general} we prove the more nuanced result which shows that the optimal $M$ which minimizes the loss on the full $T$-length context achieves length generalization in the sense that it achieves small loss even when only allowed to use context length $L$. Combining these two lemmas gives the proof of Theorem~\ref{thm:general_length_gen}.
\begin{proof}[Proof of Theorem~\ref{thm:general_length_gen}] Let
\begin{equation*}
    M_T^* \defeq \min_{M^* \in \K_{r}} \sum_{t=1}^T \ell_t(M^* , T)
\end{equation*}
and observe that
    \begin{align}
    \label{eq:silly}
        \min_{M^* \in \K_{r}} \sum_{t=1}^T \ell_t(M^* , L) \leq \sum_{t=1}^T \ell_t(M^*_{T} , L).
    \end{align}
    Combining this with Lemma~\ref{lemma:ogd_regret_general} and Lemma~\ref{lemma:length_generalization_general}, we conclude
    \begin{align*}
        \sum_{t = 1}^T \ell_t(M_t, L) & \leq \min_{M^* \in \K_{r}} \sum_{t=1}^T \ell_t(M^* , L) + 12 k^{3/2} \mtrueconst^2 \log^p(T)  \sqrt{T} & \text{OGD Regret Lemma~\ref{lemma:ogd_regret_general}} \\
        & \leq \sum_{t=1}^T \ell_t(M^*_{T} , L) + 12 k^{3/2} \mtrueconst^2 \log^p(T)  \sqrt{T} & \text{Eq.~\ref{eq:silly}} \\
         & \leq \sum_{t=1}^T \ell_t(M^*_{T} , T) + ( 12 k^{3/2} \mtrueconst^2 \log^p(T)  + 8 \norm{C}^2 \norm{B}^2 )\sqrt{T} & \text{Length Generalization Lemma~\ref{lemma:length_generalization_general}} \\
       & = \min_{M^* \in \K_{\mtrueconst}} \sum_{t=1}^T \ell_t(M , T) + ( 12 k^{3/2} \mtrueconst^2 \log^p(T)  + 8 \norm{C}^2 \norm{B}^2 )\sqrt{T}. & \text{Definition of $M_T^*$}
    \end{align*}
\end{proof}

\subsection{OGD Regret for Generalized Spectral Filtering}
\label{appendix:length_gen_ogd}

\begin{lemma}
\label{lemma:ogd_regret_general}
Suppose the input $u_{1:t}$ satisfies $\norm{u_t}_{2} \leq 1$. Suppose the true output $y_t$ evolves such that for some polynomial $p_t(y_{t-1:1})$ there exists some $\mtrue \in \K_{\mtrueconst}$
\begin{equation*}
    y_t   = p_t(y_{t-1:1}) + \sum_{i = 1}^T \mtrue_i u_{t-1:0} v_i,
\end{equation*}
and for
\begin{equation*}
    E_{m,T} \defeq \sum_{i = k+1}^T \mtrue_i u_{t-1:0}v_i,
\end{equation*}
we have $\norm{E_{m,T}} \leq 1$. Further suppose $v_1, \dots, v_k$ satisfy $\norm{v_i}_1 \leq c_i \log^p(T)$. Let 
\begin{equation*}
    \ell_t(M, L) \defeq \norm{y_t - p_t(y_{t-1:1}) - \sum_{i = 1}^k M_i u_{t-1:t-L} v_i}^2,
\end{equation*}
Then if $M^t$ are the iterates of Algorithm~\ref{alg:general_sf}
\begin{equation*}
    \sum_{t = 1}^T \ell_t(M^t, L) - \min_{M^* \in \K_{r}} \sum_{t=1}^T \ell_t(M^* , L) \leq  12 k^{3/2} \mtrueconst^2 \log^p(T) \sqrt{T}.
\end{equation*}
\end{lemma}

\begin{proof}[Proof of Lemma~\ref{lemma:ogd_regret_general}]
\label{proof:proof_ogd_general}
This proof is a near copy of the proof in \cite{hazan2017learning}, the difference is that we derive several equations that we will use later and we handle the varying context length. 

Let $G = \max_{t \in [T]} \norm{\nabla_M \ell_t(M_t, L)}$ and let $D = \max_{M_1, M_2 \in \K_{r}} \norm{M_1 - M_2}$. By Theorem A.1 from \cite{hazan2022introduction}, 
    \begin{equation*}
        \sum_{t = 1}^T \ell_t(M^t, L) -  \min_{M^* \in \K_{r}}  \sum_{t = 1}^T \ell_t(M^*, L) \leq \frac{3}{2} GD \sqrt{T}.
    \end{equation*}
    Therefore it remains to bound $G$ and $D$.

First we bound $D$. By definition of $\K_{\mtrueconst}$, we have that for any $M \in \K_{\mtrueconst}$, 
\begin{equation*}
    \norm{M_i} \leq \mtrueconst.
\end{equation*}
Therefore, we also have that
\begin{equation*}
    \norm{M} \leq \sqrt{k} \mtrueconst.
\end{equation*}
Therefore 
\begin{equation*}
    D \defeq \max_{M, M' \in \K_r} \norm{M - M' } \leq 2 \sqrt{k} \mtrueconst.
\end{equation*}
Next we bound the gradient norm $G$. Using the definition of $\K_{\mtrueconst}$, 
    \begin{align*}
        \max_{M \in \K_{\mtrueconst}} \max_{i \in [k]} \norm{M_i} \leq \mtrueconst.
    \end{align*}
    We bound the gradient norm as follows,
    \begin{align*}
        \norm{ \nabla_{M_j} \ell_t(M, L) }  
        & = \norm{ 2 \left( \sum_{i=1}^k \mtrue_i u_{t-1:0} v_i + E_{m,T} - \sum_{i=1}^k M_i u_{t-1:t-L} v_i \right)   \left(  u_{t-1:t-L} v_j \right)^{\top}} \\
        & \leq 2 \left( \sum_{i = 1}^k \norm{\mtrue_i} \norm{u_{t-1:0}}_{\infty} \norm{v_i}_1  + \norm{E_{m,T}} + \sum_{i = 1}^k \norm{M_i} \norm{u_{t-1:t-L}}_{\infty} \norm{v_i}_1 \right) \norm{u_{t:t-L}}_{\infty} \norm{v_j}_1 \\
        & \leq 2 \left( 1 + \norm{E_{m,T}} \right) \sum_{i = 1}^k \max_{M \in \K_{\mtrueconst}} \norm{M_i} \cdot \norm{u_{t-1:0}}_{\infty}^2 \cdot \norm{v_i}_1^2  \\
        & \leq 4 k \mtrueconst \log^p(T). 
    \end{align*}
    Putting everything together we have
    \begin{align*}
        \sum_{t = 1}^T \ell_t(M_t, L) -  \min_{M^* \in \K_{r}}  \sum_{t = 1}^T \ell_t(M^*, L) &  \leq \frac{3}{2} \left( 4 k \mtrueconst \log^p(T) \right) \left( 2 \sqrt{k} \mtrueconst \right) \sqrt{T} \\
        & = 12 k^{3/2} \mtrueconst^2 \log^p(T) \sqrt{T}.
    \end{align*}
\end{proof}

\subsection{Length Generalization on the Best Optimizer in Hindsight}
\begin{lemma}
\label{lemma:length_generalization_general}
Let input $u_{(t-1):0}$, $\{ v_i \}_{i = 1}^k$, $p_t(\cdot)$, and $\ell_t(M, L)$ all be as defined in Algorithm~\ref{alg:general_sf}. Suppose the input $u_{(t-1):0}$ is such that $ \sum_{t = 0}^{T-1} (T-t) u_t u_t^{\top} \succeq (2 \norm{C} \norm{B}/\sqrt{T})I$, $\{ v_i \}_{i = 1}^k$ is orthonormal with $\norm{v_i}_1 \leq \log^p(T)$, and that there exists some $\mtrue$ such that
\begin{equation*}
    y_t - p_t(y_{t-1:1})  = \sum_{i = 1}^T \mtrue_i u_{t-1:0} v_i = \sum_{i = 1}^{\ell_1} \mtrue_i u_{t-i} + \sum_{i = 1}^{t-\ell_1-1} C A^i h(A) B u_{t - \ell_1 - i}, 
\end{equation*}
where
\begin{equation*}
    \norm{\sum_{i = k +1}^T \mtrue_{i} u_{t-1:t-L} v_i} \leq \norm{C} \norm{B}/T,
\end{equation*}
and
\begin{equation*}
    \max_{\alpha(A)} \left \{ h(\alpha) \alpha^{L- \ell_1 - 1} (1 - \alpha^{T - L + 1}) (1 -\alpha)^{-1} \right \} \leq \frac{1}{T^{1/4}}.
\end{equation*}
    Let 
    \begin{equation*}
        M_{T}^* \defeq \argmin_{M \in \K_{r}} \sum_{t = 1}^T \ell_t(M, T).
    \end{equation*}
    Then for $T \geq (4 k  \log^p(T)/ \norm{C} \norm{B} )^{4}$, the loss with context $L$ well approximates the loss with context $T$ on $M_{T}^*$,
    \begin{equation*}
        \abs{ \sum_{t = 1}^T \ell_t(M_{T}^*, L) - \ell_t(M_{T}^*, T) } \leq 8 \norm{C}^2 \norm{B}^2 \sqrt{T}.
    \end{equation*}
\end{lemma}
The proof of Lemma~\ref{lemma:length_generalization_general} requires two key helper lemmas which we develop in the following subsections. The first is Lemma~\ref{lemma:minimizing_loss_is_recovery_general} which establishes that when $y_{1:t}$ evolves as a noiseless LDS and if the input $u_{1:t}$ is sufficiently well-conditioned, then the minimizer for $\sum_{t = 1}^T \ell_t(M, T)$ approximately recovers a collection of matrices (we denote as $\mtrue$) which is generated by the true linear dynamical system. The second key helper Lemma is Lemma~\ref{lemma:general_ltM_L} which establishes that an algorithm which uses the collection of matrices that are generated by the true linear dynamical system, i.e. $\mtrue$, is able to achieve small loss even when restricted to a small context-length $L << T$. The proof of Lemma~\ref{lemma:length_generalization_general} combines these two insights to establish that this implies that the minimizer for $\sum_{t = 1}^T \ell_t(M, T)$ also achieves small loss even when restricted to small context-length $L$.
\begin{proof}[Proof of Lemma~\ref{lemma:length_generalization_general}]
First we show that $\mtrue$ is a $(\norm{C}^2 \norm{B}^2/T)$-approximate minimizer to $\sum_{t = 1}^T \ell_t(M, T)$. Indeed,
\begin{align*}
    \sum_{t = 1}^T \ell_t(\mtrue, T) & =  \sum_{t = 1}^T \norm{ y_t - p_t(y_{t-1:1}) -  \sum_{i = 1}^k \mtrue_i u_{t-1:0} v_i}^2 \\
    & =  \sum_{t = 1}^T \norm{ \sum_{i = k +1}^T \mtrue_i u_{t-1:0} v_i}^2 \\
    & \leq \norm{C}^2 \norm{B}^2/T.
\end{align*}
By assumption $ \sum_{t = 0}^{T-1} (T-t) u_t u_t^{\top} \succeq (2 \norm{C} \norm{B}/\sqrt{T})I$. Therefore, by Lemma~\ref{lemma:minimizing_loss_is_recovery_general} with $\epsilon = \norm{C} \norm{B}/\sqrt{T}$ we have
\begin{equation*}
    M_T^* \in \mathcal{B}_{ \norm{C} \norm{B}/\sqrt{T}} \left( \mtrue \right).
\end{equation*}
Since we assumed $T \geq (4 k  \log^p(T)/ \norm{C} \norm{B} )^4$ we have
\begin{equation*}
    \norm{C} \norm{B}/\sqrt{T} \leq \norm{C}^2 \norm{B}^2 /(4 k  T^{1/4} \log^p(T)).
\end{equation*}
Therefore by Lemma~\ref{lemma:general_ltM_L} 
we have
\begin{equation*}
    \sum_{t = 1}^T \ell_t(M_T^*, L ) \leq 4 \norm{C}^2 \norm{B}^2 \sqrt{T}.
\end{equation*}
Moreover note that 
\begin{equation*}
\label{eqn:bound_on_ltMT}
    0 \leq \ell_t(M_T^*, T) \leq \ell_t(\mtrue, T) \leq \norm{C}^2 \norm{B}^2/T^2.
\end{equation*}
Combining these we conclude,
\begin{equation*}
    \abs{ \sum_{t = 1}^T \ell_t(M_T^*, L ) - \sum_{t = 1}^T \ell_t(M_T^*, T) } \leq 4 \norm{C}^2 \norm{B}^2 \sqrt{T} + \norm{C}^2 \norm{B}^2/T \leq 8 \norm{C}^2 \norm{B}^2 \sqrt{T}.
\end{equation*}
\end{proof}

\subsubsection{Minimization is Recovery}
\begin{lemma}
\label{lemma:minimizing_loss_is_recovery_general}
    Suppose $\sum_{t = 0}^{T-1} (T-t) u_t u_t^{\top} \succeq 2 \epsilon I$ and $\{ v_i \}_{i = 1}^k$ is orthonormal. Then there is a unique point $M^*$ which minimizes the function $\sum_{t = 1}^T \ell_t(M, T)$ from Algorithm~\ref{alg:general_sf}. Moreover, suppose some $k$ satisfies 
    \begin{equation*}
        \sum_{t = 1}^T \ell_t(M, T) \leq \epsilon^2.
    \end{equation*}
    Then there is a matrix $E_M$ such that $\norm{E_M} \leq \epsilon$ and 
    \begin{equation*}
        M^* = M + E_M.
    \end{equation*}
\end{lemma}

\begin{proof}
    For convenience, let $X_t$ be the $kd_{\mathrm{in}}$-dimensional vector which stacks the filters,
    \begin{equation*}
        X_t = \begin{bmatrix}
             u_{t-1:t-T} v_1   \\
             u_{t-1:t-T} v_2\\
            \vdots \\
            u_{t-1:t-T} v_k 
        \end{bmatrix} = \begin{bmatrix}
             u_{t-1:0} v_1   \\
             u_{t-1:0} v_2\\
            \vdots \\
            u_{t-1:0} v_k 
        \end{bmatrix},
    \end{equation*}
    where the second inequality holds since we only consider $t \leq T$.
    Assume $k$ is written as $M =\begin{bmatrix} M_{1} & M_{2} & \dots & M_{k} \end{bmatrix} \in \reals^{d_{\mathrm{out}} \times kd_{\mathrm{in}}}$ and let $Y_t = y_t - p_t(y_{t-1:1})$. Let $Y = \begin{bmatrix}
        Y_1 & Y_2 & \dots & Y_T
    \end{bmatrix}$ and $X = \begin{bmatrix}
        X_1 & X_2 & \dots & X_T
    \end{bmatrix}$. Then we can express the loss as
    \begin{equation*}
    f(M) \defeq \sum_{t = 1}^T \ell_t(M, T) = \norm{Y - M X}^2. 
    \end{equation*}
    Note that this function is twice differentiable and 
    \begin{equation*}
        \nabla^2_M f(M) = XX^{\top}.
    \end{equation*}
    Therefore, if $\lambda_{\textrm{min}}\left( XX^{\top}\right)\geq \mu$ we have that $f(M)$ is $\mu$-strongly convex. Then if $M^*$ is the optimum of $f(M)$ we have
    \begin{equation*}
        f(M) \geq f(M^*) + \frac{\mu}{2} \norm{M - M^*}^2, \textrm{ or equivalently, } \norm{M - M^*} \leq \frac{2}{\mu} \left(f(M) - f(M^*)\right).
    \end{equation*}
    Now suppose $k$ is such that $f(M) \leq \epsilon^2$. Then since $f(\cdot) \geq 0$ we have
    \begin{equation*}
        \norm{M - M^*} \leq 2 \epsilon^2/\mu.
    \end{equation*}
    Therefore we can write
    \begin{equation}
    \label{eqn:normEM}
        M^* = M + E_{M^*} \textrm{ where } \norm{E_{M^*}} \leq 2 \epsilon^2/\mu.
    \end{equation}    
Next we must understand the eigenvalues of $XX^{\top}$ and how they relate to the input $u_{T:1}$. For notational convenience, let $U = u_{T:1}$ and let $D_t$ denote the block-diagonal $T \times T$ matrix
\begin{equation*}
    D_t \defeq \begin{bmatrix}
        0_{T-t \times T-t} &  \\
        & I_{t}
    \end{bmatrix}.
\end{equation*}
Finally, let \begin{equation*}
    V = \begin{bmatrix}
    v_1 \\
    v_2 \\
    \vdots \\
    v_k
    \end{bmatrix} \in \R^{Tm \times 1}
\end{equation*}
 Then we have $X_t = \left( I_k \otimes UD_t \right) V$ and we observe
\begin{align*}
    XX^{\top} & = \sum_{t = 1}^T X_t X_t^{\top} = \sum_{t = 1}^T \left(  (I_{k} \otimes UD_t) V  \right) \left( (I_{k} \otimes UD_t) V  \right)^{\top} \\
   & = \sum_{t = 1}^T (I_{k} \otimes UD_tU^{\top}) \\
   & =  I_{k} \otimes U \left( \sum_{t = 1}^T  D_t \right) U^{\top}.
\end{align*}
Observe that 
\begin{equation*}
    \sum_{t = 1}^T  D_t  = \mathrm{diag} \left( \begin{bmatrix}
        1 & 2 & \dots & T
    \end{bmatrix} \right).
\end{equation*}
Using this we can further refine
\begin{align*}
     U \left( \sum_{t = 1}^T  D_t \right) U^{\top} & = \sum_{t = 0}^{T-1} (T-t) u_t u_t^{\top}.
\end{align*}
By assumption, this matrix has minimum eigenvalue bounded below by $2 \epsilon$. Therefore $\lambda_{\textrm{min}}(XX^{\top}) \geq 2 \epsilon$. Plugging this value in for $\mu$ in Eq.~\ref{eqn:normEM} concludes the proof.

\end{proof}

\subsubsection{Uniform Length Generalization Around LDS Generated Solutions}
The following lemma shows that any $k$ in an (appropriately defined) $\epsilon$-ball around $\mtrue$ obtains length generalization in the sense that it achieves $O(\sqrt{T})$ $L$-context-length-limited loss $\sum_{t = 1}^T \ell_t(\cdot, L)$.
\begin{lemma}
\label{lemma:general_ltM_L}
Suppose $y_t$ evolves as a noiseless $(A, B, C, I)$-LDS with input $u_t$. Suppose $p_t(\cdot)$ and $\mtrue$ is such that 
\begin{equation*}
    y_t - p_t(y_{t-1:1})  = \sum_{i = 1}^T \mtrue_i u_{t-1:0} v_i = \sum_{i = 1}^{\ell_1} \mtrue_i u_{t-i} + \sum_{i = 1}^{t-\ell_1-1} C A^i h(A) B u_{t - \ell_1 - i}. 
\end{equation*}
Suppose for a given $k > 0$,
\begin{equation*}
    \norm{\sum_{i = k+1}^T \mtrue_{i} u_{t-1:t-L} v_i} \leq \frac{\norm{C} \norm{B}}{T}.
\end{equation*}
Suppose
\begin{equation*}
    \max_{\alpha(A)} \left \{ h(\alpha) \alpha^{L- \ell_1 - 1} (1 - \alpha^{T - L + 1}) (1 -\alpha)^{-1} \right \} \leq \frac{1}{T^{1/4}}.
\end{equation*}
If 
\begin{equation*}
     \delta \leq \frac{1}{4m } \frac{\norm{C}^2 \norm{B}^2}{T^{1/4} \log^p(T)},
\end{equation*}
then  we have for any $M \in \mathcal{B}_{\delta}(\mtrue)$
\begin{equation*}
    \sum_{t = 1}^T \ell_t(M, L) \leq  4 \norm{C}^2 \norm{B}^2 \sqrt{T}.
\end{equation*}
\end{lemma}
\begin{proof}[Proof of Lemma~\ref{lemma:general_ltM_L}]
Let $M = \mtrue + E_M$, where $\norm{E_M} \leq \delta$. By definition,
\begin{align*}
    \ell_t(\mtrue + E_M, L) & = \norm{ y_t - p_t(y_{t-1:1}) - \sum_{i = 1}^k \left( \mtrue + E_M \right)_i u_{t-1:t-L} v_i}^2 \\
    & = \norm{ y_t - p_t(y_{t-1:1}) - \sum_{i = 1}^k  \mtrue_i  u_{t-1:t-L} v_i -  \sum_{i = 1}^k  E_{M_i}  u_{t-1:t-L} v_i}^2 \\
    & \leq \norm{ y_t - p_t(y_{t-1:1}) - \sum_{i = 1}^k  \mtrue_i  u_{t-1:t-L} v_i}^2  \\
    & \qquad + 2 \norm{ y_t - p_t(y_{t-1:1}) - \sum_{i = 1}^k  \mtrue_i  u_{t-1:t-L} v_i} \norm{\sum_{i = 1}^k  E_{M_i}  u_{t-1:t-L} v_i} \\
    & \qquad + \norm{\sum_{i = 1}^k  E_{M_i}  u_{t-1:t-L} v_i}^2.
\end{align*}

Observe that
\begin{align*}
    \norm{ \sum_{i = 1}^k  E_{M_i}  u_{t-1:t-L} v_i } & \leq \sum_{i = 1}^k \norm{E_{M_i}} \norm{u_{t-1:t-L}}_{\infty} \norm{v_i}_1  \leq k \delta \log^p(T).
\end{align*}
For the remainder of the proof we work towards bounding $\norm{ y_t - p_t(y_{t-1:1}) - \sum_{i = 1}^k  \mtrue_i  u_{t-1:t-L} v_i}$.
    We replace $y_t - p_t(y_{t-1:1})$ with $\sum_{i = 1}^T \mtrue_i u_{t-1:0} v_i $ and we replace $\sum_{i = 1}^k \mtrue_{i} u_{t-1:t-L} v_i$ with $\sum_{i = 1}^T \mtrue_{i} u_{t-1:t-L} v_i - \sum_{i = k+1}^T \mtrue_{i} u_{t-1:t-L} v_i$ to get
\begin{align*}
 \norm{ y_t - p_t(y_{t-1:1}) - \sum_{i = 1}^k \mtrue_{i} u_{t-1:t-L} v_i}^2 &=   \norm{ \left( \sum_{i = 1}^T \mtrue_i u_{t-1:0} v_i  \right) - \left( \sum_{i = 1}^T \mtrue_{i} u_{t-1:t-L} v_i - \sum_{i = k+1}^T \mtrue_{i} u_{t-1:t-L} v_i \right) }^2 \\
    & \leq  \norm{  \sum_{i = 1}^T \mtrue_i (u_{t-1:0} - u_{t-1:t-L} )v_i  }^2   \\
    & \qquad + 2 \norm{  \sum_{i = 1}^T \mtrue_i (u_{t-1:0} - u_{t-1:t-L}) v_i  } \norm{\sum_{i = k+1}^T \mtrue_{i} u_{t-1:t-L} v_i } \\
    & \qquad + \norm{\sum_{i = k+1}^T \mtrue_{i} u_{t-1:t-L} v_i }^2.
\end{align*}
Next we note that  $ \norm{\sum_{i = k+1}^T \mtrue_{i} u_{t-1:t-L} v_i}$ is assumed to be at most $\norm{C} \norm{B}/T$ and so we now focus on bounding the norm:
\begin{equation}
\label{eqn:normbound}
    \norm{ \sum_{i = 1}^T \mtrue_i (u_{t-1:0} - u_{t-1:t-L}) v_i  }.
\end{equation}
Towards bounding Eq.~\ref{eqn:normbound}, assume $L > \ell_1$ so that 
\begin{align*}
    \sum_{i = 1}^T \mtrue_i (u_{t-1:0} - u_{t-1:t-L}) v_i & = \sum_{i = L-\ell_1 + 1}^{t-\ell_1-1} C A^i h(A) B u_{t - \ell_1 - i} \\
    & = \sum_{i = L-\ell_1 + 1}^{t-\ell_1-1} \sum_{j = 1}^{d_A} \alpha_j^i h(\alpha_j) C_j B_j^{\top} u_{t - \ell_1 - i}.
\end{align*}
Then
\begin{align*}
    \norm{\sum_{i = L-\ell_1 + 1}^{t-\ell_1-1} C A^i h(A) B u_{t - \ell_1 - i}} & \leq \max_{j \in [d_A]}  \alpha_j^i h(\alpha_j) \sum_{i = L-\ell_1 + 1}^{t-\ell_1-1} \norm{C_j B_j^{\top} u_{t - \ell_1 - i}} \\
    & \leq \max_{\alpha(A)} \sum_{i = L-\ell_1 + 1}^{t-\ell_1-1}  \alpha^i h(\alpha) \norm{C} \norm{B}.
\end{align*}
Next we have
\begin{align*}
     \left( \max_{\alpha(A)} \sum_{i = L-\ell_1 + 1}^{t-\ell_1-1}  \alpha^i h(\alpha)  \right) & \leq h(\alpha) \alpha^{L- \ell_1 -1} \sum_{i = 0}^{T - L} \alpha^i \\
     & = h(\alpha) \alpha^{L- \ell_1 - 1} \frac{1 - \alpha^{T-L + 1}}{1 - \alpha} \\
     & \leq T^{-1/4},
\end{align*}
where the last inequality holds by assumption. Therefore Eq.~\ref{eqn:normbound} is at most 
\begin{align*}
    \norm{ \sum_{i = 1}^T \mtrue_i (u_{t-1:0} - u_{t-1:t-L}) v_i  } \leq \norm{C}\norm{B} T^{-1/4}.
\end{align*}
Then we have
\begin{align*}
     \norm{ y_t - p_t(y_{t-1:1}) - \sum_{i = 1}^k \mtrue_{i} u_{t-1:t-L} v_i}^2 &\leq \frac{\norm{C}^2 \norm{B}^2}{T^{1/2}} + 2 \frac{\norm{C}^2 \norm{B}^2}{T^{3/4}} + \frac{\norm{C}^2 \norm{B}^2}{T^2} \leq 2 \frac{\norm{C}^2 \norm{B}^2}{T^{1/2}}.
\end{align*}
Finally we conclude
\begin{align*}
    \ell_t(\mtrue + E_M, L) & \leq 2 \frac{\norm{C}^2 \norm{B}^2}{T^{1/2}} + 2 \left( 2 \frac{\norm{C}^2 \norm{B}^2}{T^{1/2}} \right)^{1/2} \left( k  \delta \log^p(T) \right) + \left( k  \delta \log^p(T) \right)^2 \\
    & \leq 4 \frac{\norm{C}^2 \norm{B}^2}{T^{1/2}},
\end{align*}
where the last inequality holds since we assumed
\begin{equation*}
     \delta \leq \frac{1}{4m } \frac{\norm{C}^2 \norm{B}^2}{T^{1/4} \log^p(T)}.
\end{equation*}
\end{proof}

\section{Length Generalization for Vanilla Spectral Filtering}
\label{appendix:vanilla_proof}

The proof of Theorem~\ref{thm:lengthgeneralization} ultimately comes from Theorem~\ref{thm:general_length_gen} and its proof in Appendix~\ref{appendix:general}. Theorem~\ref{thm:general_length_gen} abstracts the necessary assumptions needed to obtain a length generalization guarantee. In Lemma~\ref{lemma:application_of_vanilla_to_general} we prove that Algorithm~\ref{alg:ogd_short_length} satisfies these assumptions. 
\begin{proof}[Proof of Theorem~\ref{thm:lengthgeneralization}]
    By Lemma~\ref{lemma:application_of_vanilla_to_general} and the assumptions made in the statement of Theorem~\ref{thm:lengthgeneralization}, we may apply Theorem~\ref{thm:general_length_gen} to Algorithm~\ref{alg:ogd_short_length} to get that \begin{equation*}
    \sum_{t = 1}^T \ell_t(M^t, L) - \min_{M^* \in \K_{\norm{C} \norm{B}}} \sum_{t=1}^T \ell_t(M^* , T) \leq \left( 12 k^{3/2} \norm{C}^2  \norm{B}^2 \log(T) + 8 \norm{C}^2 \norm{B}^2 \right) \sqrt{T}.
\end{equation*}
\end{proof}

\begin{lemma}[Length Generalization for Vanilla Spectral Filtering]
\label{lemma:application_of_vanilla_to_general}
Recall that in Algorithm~\ref{alg:ogd_short_length} we define
\begin{equation*}
\mu_{\alpha} \defeq (\alpha - 1) \begin{bmatrix}  1 & \alpha &  \dots & \alpha^{T-1} \end{bmatrix}^{\top} \in \reals^{T-1}
\end{equation*}
and $H_{T-1} = \int_{\alpha \in [0,1]} \mu_{\alpha} \mu_{\alpha}^{\top} d \alpha$ and we let $\phi_1, \dots, \phi_{T-1}$ be the orthonormal eigenvectors of $H_{T-1}$ with eigenvalues $\sigma_1, \dots, \sigma_{T-1}$. 
    Algorithm~\ref{alg:ogd_short_length} is equivalent to Algorithm~\ref{alg:general_sf} with the following:
    \begin{enumerate}[label=(\alph*)]
    \item $p_t(y_{t-1:1}) = y_{t-1}$
    \item $v_1 = e_1$
    \item $v_i = (0, \sigma_{i-1}^{1/4} \phi_{i-1})$ for $i = 2, \dots, T$
\end{enumerate}
Define $\mtrue$ as follows:
\begin{equation*}
    \mtrue_1 \defeq CB,
\end{equation*}
and for $i \geq 2$
\begin{equation*}
    \mtrue_i  \defeq \sum_{n = 1}^{d_A} \sigma_{i-1}^{-1/4} \phi_{i-i}^{\top} \mu_{\alpha_n}  (C_n B_n^{\top}).
\end{equation*}
Then the following properties hold 
\begin{enumerate}
\item For $h(A) = A-I$ and $\ell_1 = 1$
\begin{equation*}
     y_t - p_t(y_{t-1:1})  =  \sum_{i = 1}^{\ell_1} \mtrue_i u_{t-i} + \sum_{i = 1}^{t-\ell_1} C A^i h(A) B u_{t - \ell_1 - i}.
\end{equation*}
\item $y_t - p_t(y_{t-1:1}) = \sum_{i = 1}^T \mtrue_i u_{t-1:1} v_i$.
\item For $k = \Omega( \log(T d_A \norm{C} \norm{B}/\epsilon))$,
\begin{equation*}
    \norm{\sum_{i = k+1}^T \mtrue_i u_{t-1:1} v_i} \leq \epsilon/T.
\end{equation*}
\item For any $i \in [T]$
\begin{equation*}
    \norm{\mtrue_i} \leq \norm{C} \norm{B}.
\end{equation*}
\item  For any $i \in [T]$, $\norm{v_i}_1 \leq \log(T)$ and $\{v_i \}_{i \in [T]}$ are orthonormal.
\item Finally if the spectrum of $A$ lies in the interval \begin{equation*}
    \left[ 0, 1 - \frac{ \log(T)}{2(L-2)} \right] \cup \left[  1 - \frac{1}{2 T^{5/4}}, 1  \right],
\end{equation*}
then
    \begin{equation*}
    \max_{\alpha(A)} \left \{ \abs{ h(\alpha) \alpha^{L- \ell_1 - 1} (1 - \alpha^{T - L + 1}) (1 -\alpha)^{-1}} \right \} \leq \frac{1}{T^{1/4}}.
\end{equation*}  
\end{enumerate}

\end{lemma}

\begin{proof}
    Points $(a)-(c)$ are evident by definition of Algorithm~\ref{alg:ogd_short_length}. Now suppose $y_t$ evolves as an LDS. By definition, there exist matrices $(A, B, C, D)$ such that
\begin{equation*}
    y_{t} = \sum_{i = 1}^t C A^{i-1} B u_{t-i},
\end{equation*}
where we assume $D = I$ and $A$ is diagonal without loss of generality.  Let $\alpha_1, \dots, \alpha_{d_A}$ denote the eigenvalues of $A$.  and let $u_{t:0}$ be the $d_{\mathrm{in}} \times T$ (padded) matrix $u_{t:0}  = \begin{bmatrix}
        u_t & 
        u_{t-1} & 
        \dots & 
        u_0 & 0 \end{bmatrix}$.
Then we have
\begin{align*}
    y_t - y_{t-1} & = \sum_{i = 1}^t C A^{i-1}  B u_{t-i} - \sum_{i = 1}^{t-1} C  A^{i-1}  B u_{t-1-i}  \\
    & =C B u_{t-1} + \sum_{i = 1}^{t-1}  C  \left( A^{i} - A^{i-1} \right)  B u_{t-1-i}.
\end{align*}
We pause here to note this proves $(1)$. We continue rearranging the equation to finish the derivation of $(2)$. 
\begin{align*}
 y_t - y_{t-1} & =C B u_{t-1} + \sum_{i = 1}^{t-1}  C  \left( A^{i} - A^{i-1} \right)  B u_{t-1-i} \\
    & = C B u_{t-1}  +      \sum_{n = 1}^{d_A}  
 C e_n e_n^{\top} B   \sum_{i = 1}^{t-1}  \left( \alpha_n^{i} - \alpha_n^{i-1} \right)  u_{t-1-i} \\
 & =  C B u_{t-1} + \sum_{n = 1}^{d_A}  
 (C_n B_n^{\top}) u_{(t-2):0} \mu_{\alpha_j}.
\end{align*}
Observe that
\begin{equation*}
    \sum_{i = 1}^{T-1} \phi_i  \phi_i^{\top} = \id.
\end{equation*}
Using this we have,
\begin{align*}
    y_t - y_{t-1} 
 & =   C B u_{t-1} + \sum_{n = 1}^{d_A}  
 (C_n B_n^{\top}) u_{(t-2):0} \mu_{\alpha_n} \\
 & =  C B u_{t-1} + \sum_{n = 1}^{d_A}  
(C_n B_n^{\top})  u_{(t-2):0}  \left( \sum_{i = 1}^{T}  \phi_{i} \phi_{i}^{\top}  \right) \mu_{\alpha_n}  \\
  & =  C B u_{t-1} + \sum_{i = 1}^{T}  \sum_{n = 1}^{d_A} \phi_{i}^{\top} \mu_{\alpha_n} 
 (C_n B_n^{\top}) u_{(t-2):0}  \phi_{i} .
\end{align*}
Recalling the definition of $\mtrue$ and $v_i = \sigma_{i-1}^{1/4} \phi_{i-1}$ we therefore have established $(2)$:
\begin{equation*}
    y_t - y_{t-1} = \mtrue_1 u_{(t-1):0} e_1 + \sum_{i = 2}^{T-1} \mtrue_i u_{(t-1):0} v_i.
\end{equation*}
Next we aim to prove $(3)$. We consider 
\begin{equation*}
    \norm{ \sum_{i = k+1}^{T} \mtrue_{i} u_{(t-2):0}  v_{i}}. 
\end{equation*}
By Lemma 13.4 in \cite{hazan2022introduction} there is some universal constant $c'$ such that,
\begin{equation*}
   \max_{\alpha \in [0,1]} \abs{\phi_{i}^{\top} \mu_{\alpha}} \leq c' T^2 \exp(-i/\log(T)).
\end{equation*}
So,
\begin{align*}
    \norm{ \mtrue_i u_{(t-2):0} v_i }  & = \norm{ \sum_{n = 1}^{d_A} \sigma_{i-1}^{-1/4} \phi_{i-i}^{\top} \mu_{\alpha_n}  (C_n B_n^{\top}) u_{(t-2):0}  \left( \sigma_{i-1}^{1/4} \phi_{i-1} \right) } \\
    & = \norm{ \sum_{n = 1}^{d_A}  \phi_{i-i}^{\top} \mu_{\alpha_n}  (C_n B_n^{\top}) u_{(t-2):0}  \phi_{i-1}  } \\
    & \leq  d_A  ( c' T^2 \exp(-(i-1)/\log(T)))  \norm{ C_n B_n^{\top}} \norm{ \phi_{i-1} }_1  \\
    & \leq c' d_A T^{3/2} \exp(-(i-1)/\log(T)))\norm{C} \norm{B}.
\end{align*}
Therefore, 
\begin{align*}
  \norm{\sum_{i = k+1}^{T} \mtrue_{i} u_{(t-1):0}  \phi_{i}} &  \leq c' d_A T^{5/2} \exp(-k/\log(T)))\norm{C} \norm{B}
\end{align*}
Therefore as long as 
\begin{equation*}
  k  \geq  \log(T) \log \left( \frac{T^{5/2} c' d_A \norm{C} \norm{B} }{\epsilon}\right),
\end{equation*}
then 
\begin{equation*}
    \norm{\sum_{i = k+1}^{T} \mtrue_{i} u_{(t-1):0}  \phi_{i}} \leq \frac{\epsilon}{T}.
\end{equation*}
Next we note that the proof of $(4)$ that $\norm{\mtrue_i} \leq \norm{C} \norm{B}$ is proven in Lemma D.1 of \cite{hazan2017learning}. Similarly, the proof of $(5)$ that $\norm{v_i}_1 \leq \log(T)$ is proven by Lemma~\ref{lemma:phi_is_sparse} from \cite{hazan2017learning}.
Finally we prove $(6)$. Since $h(\alpha) = \alpha - 1$ and $\ell_1 =1 $, we have
\begin{align}
\label{eqn:alpha_terms}
    \max_{\alpha(A)} \left \{ \abs{ h(\alpha) \alpha^{L- \ell_1 - 1} (1 - \alpha^{T - L + 1}) (1 -\alpha)^{-1}} \right \} & = \max_{\alpha(A) } \alpha^{L-2} (1 - \alpha^{T-L + 1}).
\end{align}
To bound Eq.~\ref{eqn:alpha_terms}, consider the case where $\alpha$ is bounded away from $1$. Suppose $\alpha = 1 - \delta$, then
\begin{align*}
    (1 - \delta)^{L-2} \leq \frac{1}{T^p} \iff \log\left( \frac{1}{1 - \delta} \right)  \geq \frac{p \log(T)}{L-2}.
\end{align*}
Observe that for $\delta \in [0, 1]$, $\log(1/(1-\delta)) \geq \delta/2$. Therefore, if
\begin{align*}
    \delta \geq \frac{2p \log(T)}{L-2},
\end{align*}
we are guaranteed that $\alpha^{L-2} \leq 1/T^p$. Next consider when $\alpha$ is very close to $1$; suppose $\alpha \geq 1 - \frac{1}{T^p T}$ for $p < 1/2$. Then using that $(1 - x)^q \geq 1- 2qx$ for $x \in [0,1]$ we have
\begin{equation*}
    \alpha^{T - L+1} \geq \left( 1 - \frac{1}{T^p T} \right)^{T - L+1} \geq 1 - 2  \frac{T - L+1}{T^p T} \implies 1 - \alpha^{T - L+1} \leq 2  \frac{T - L+1}{T^p T} \leq \frac{2}{T^p}.
\end{equation*}
Plugging in $p = 1/4$ we conclude that 
\begin{equation*}
   \alpha^{L-2} (1 - \alpha^{T - L+1 }) \leq T^{-1/4} \qquad \textrm{ for any } \alpha \in \left[ 0, 1 - \frac{ \log(T)}{2(L-2)} \right] \cup \left[  1 - \frac{1}{2 T^{5/4}}, 1  \right].
\end{equation*}
\end{proof}

The following lemma comes from \cite{hazan2017learning}.
\begin{lemma}[Hazan, Singh, Zhang]
\label{lemma:phi_is_sparse}
    Let $(\sigma_j, \phi_j)$ be the $j$-th largest eigenvalue-eigenvector pair of the $T \times T$ Hankel matrix. Then,
    \begin{equation*}
        \norm{\phi_j}_1 \leq O\left( \frac{\log(T)}{\sigma_j^{1/4}}\right).
    \end{equation*}
\end{lemma}
\section{Length Generalization for Spectral Filtering Using Two Autoregressive Components}
\label{appendix:two_vanilla}
The proof of Theorem~\ref{thm:lengthgeneralization_two_auto} ultimately comes from Theorem~\ref{thm:general_length_gen} and its proof in Appendix~\ref{appendix:general}. Theorem~\ref{thm:general_length_gen} abstracts the necessary assumptions needed to obtain a length generalization guarantee. In Lemma~\ref{lemma:application_two_auto} we prove that Algorithm~\ref{alg:sf_two} satisfies these assumptions. 
\begin{proof}[Proof of Theorem~\ref{thm:lengthgeneralization_two_auto}]
    By Lemma~\ref{lemma:application_two_auto} and the assumptions made in the statement of Theorem~\ref{thm:lengthgeneralization_two_auto}, we may apply Theorem~\ref{thm:general_length_gen} to Algorithm~\ref{alg:sf_two} to get that \begin{equation*}
    \sum_{t = 1}^T \ell_t(M^t, L) - \min_{M^* \in \K_{\norm{C} \norm{B}}} \sum_{t=1}^T \ell_t(M^* , T) \leq \left( 12 k^{3/2} \norm{C}^2  \norm{B}^2 \log^2(T) + 8 \norm{C}^2 \norm{B}^2 \right) \sqrt{T}.
\end{equation*}
\end{proof}

\begin{lemma}[Length Generalization Using Two Autoregressive Components]
\label{lemma:application_two_auto}
Recall that in Algorithm~\ref{alg:sf_two} we define
\begin{equation*}
\tilde{\mu}_{\alpha,T} \defeq (\alpha - 1)^2 \begin{bmatrix}  1 & \alpha &  \dots & \alpha^{T} \end{bmatrix}^{\top} \in \reals^{T}
\end{equation*}
and
and $N_T = \int_{\alpha \in [0,1]} \tilde{\mu}_{\alpha,T} \tilde{\mu}_{\alpha,T}^{\top} d \alpha$ and we let $\tilde{\phi}_1, \dots, \tilde{\phi}_{T-2}$ be the orthonormal eigenvectors of $N_{T-2}$ with eigenvalues $\tilde{\sigma}_1, \dots, \tilde{\sigma}_{T-2}$. 
    Algorithm~\ref{alg:sf_two} is equivalent to Algorithm~\ref{alg:general_sf} with the following:
    \begin{enumerate}[label=(\alph*)]
    \item $p_t(y_{t-1:1}) = 2 y_{t-1} - y_{t-2}$
    \item $v_1 = e_1$, $v_2 = e_2$ and for $i \geq 3$, $v_i = (0, 0, \sigma_{i-2}^{1/4} \tilde{\phi}_{i-2})$ 
\end{enumerate}

Define $\mtrue$ as follows:
\begin{equation*}
    \mtrue_1 \defeq CB,
\end{equation*}
\begin{equation*}
    \mtrue_2 \defeq C(A-2I) B,
\end{equation*}
and for $i \geq 3$,  
\begin{equation*}
    \mtrue_{i}  \defeq \sum_{n = 1}^{d_A} \left( \sigma_{i}^{-1/4} \tilde{\phi}_{i}^{\top} \tilde{\mu}_{\alpha_n} \right) (C_n B_n^{\top}).
\end{equation*}
Then the following properties hold 
\begin{enumerate}
\item For $h(A) = (A-I)^2$ and $\ell_1 = 2$
\begin{equation*}
     y_t - p_t(y_{t-1:1})  =  \sum_{i = 1}^{\ell_1} \mtrue_i u_{t-i} + \sum_{i = 1}^{t-\ell_1} C A^i h(A) B u_{t - \ell_1 - i}.
\end{equation*}
\item $y_t - p_t(y_{t-1:1}) = \sum_{i = 1}^T \mtrue_i u_{t-1:1} v_i$.
\item For $k = \Omega( \log(T d_A \norm{C} \norm{B}/\epsilon))$,
\begin{equation*}
    \norm{\sum_{i = k+1}^T \mtrue_i u_{t-1:1} v_i} \leq \epsilon/T.
\end{equation*}
\item For any $i \in [T]$
\begin{equation*}
    \norm{\mtrue_i} \leq \norm{C} \norm{B}.
\end{equation*}
\item  For any $i \in [T]$, $\norm{v_i}_1 \leq \log(T)$ and $\{v_i \}_{i \in [T]}$ are orthonormal.
\item Finally if the spectrum of $A$ lies in the interval \begin{equation*}
    \left[ 0, 1 - \frac{ \log(T)}{2(L-2)} \right] \cup \left[  1 - \frac{1}{2 T^{1/4}}, 1  \right],
\end{equation*}
then
    \begin{equation*}
    \max_{\alpha(A)} \left \{ \abs{ h(\alpha) \alpha^{L- \ell_1 - 1} (1 - \alpha^{T - L + 1}) (1 -\alpha)^{-1}} \right \} \leq \frac{1}{T^{1/4}}.
\end{equation*}  
\end{enumerate}

\end{lemma}

\begin{proof}
    Suppose $y_t$ evolves as an LDS. By definition, there exist matrices $(A, B, C, D)$ such that
\begin{equation*}
    y_{t} = \sum_{i = 1}^t C A^{i-1} B u_{t-i},
\end{equation*}
where we assume $D = I$ and $A$ is diagonal without loss of generality.  Let $\alpha_1, \dots, \alpha_{d_A}$ denote the eigenvalues of $A$.  and let $u_{t:0}$ be the $d_{\mathrm{in}} \times T$ (padded) matrix $u_{t:0}  = \begin{bmatrix}
        u_t & 
        u_{t-1} & 
        \dots & 
        u_0 & 0 \end{bmatrix}$.
Then we have $(1)$:
\begin{align*}
    y_t - 2 y_{t-1} + y_{t-2} & = CB u_{t-1} + C(A-2I)Bu_{t-2} + \sum_{i = 0}^{t-3} C  A^i(A^2 - 2A + I) B u_{t-3-i}.
\end{align*}
Let $\alpha_1, \dots, \alpha_{d_A}$ denote the eigenvalues of $A$. We observe the following equality:
\begin{align*}
    \sum_{i = 0}^{t-3} C  A^i(A^2 - 2A + I) B u_{t-3-i}
    & =  \sum_{i = 0}^{t-3} C  \sum_{n = 1}^{d_A} \alpha_n^i (\alpha_n -1)^2 e_n e_n^{\top} B u_{t-3-i} \\
    & = \sum_{n = 1}^{d_A}   \left( C e_n e_n^{\top} B\right)    \sum_{i = 0}^{t-3} \alpha_n^i (\alpha_n -1)^2 u_{t-3-i} \\
    & = \sum_{n = 1}^{d_A}   \left( C_n B_n^{\top} \right)     u_{(t-3):0} \tilde{\mu}_{\alpha_n}.
\end{align*}
Observe that
\begin{equation*}
    \sum_{i = 1}^{T-2} \tilde{\phi}_i  \tilde{\phi}_i^{\top} = \id.
\end{equation*}
Using this we have,
\begin{align*}
    \sum_{i = 0}^{t-3} C  A^i(A^2 - 2A + I) B u_{t-3-i} 
 & =  \sum_{n = 1}^{d_A}   \left( C_n B_n^{\top} \right)     u_{(t-3):0} \tilde{\mu}_{\alpha_n} \\
 & = \sum_{n = 1}^{d_A}   \left( C_n B_n^{\top} \right)     u_{(t-3):0} \left(  \sum_{i = 1}^{T-2} \tilde{\phi}_i  \tilde{\phi}_i^{\top}  \right) \tilde{\mu}_{\alpha_n} \\
 & = \sum_{i = 1}^{T-2} \left( \sum_{n = 1}^{d_A}  \tilde{\phi}_i^{\top} \tilde{\mu}_{\alpha_n} \left( C_n B_n^{\top} \right)    \right)  u_{(t-3):0}  \tilde{\phi}_i \\
 & = \sum_{\ell = 3}^{T} \mtrue_{\ell} u_{(t-1):0}  v_{\ell}.
\end{align*}
Therefore we have established $(2)$. Next we aim to prove $(3)$. We consider 
\begin{equation*}
    \norm{ \sum_{i = k+1}^{T} \mtrue_{i} u_{(t-1):0}  v_{i}}. 
\end{equation*}
Combining Lemma~\ref{lemma:nu_alpha_phi_small} and Lemma~\ref{lemma:exp_decay_nu_eigs} gives us that there is some constant $c'$ such that,
\begin{equation*}
   \max_{\alpha \in [0,1]} \abs{\tilde{\phi}_{i}^{\top} \tilde{\mu}_{\alpha}} \leq c'  \exp(-i/4\log(T)).
\end{equation*}
So,
\begin{align*}
    \norm{ \mtrue_i u_{(t-1):0} v_i }  & = \norm{ \sum_{n = 1}^{d_A} \sigma_{i-1}^{-1/4} \tilde{\phi}_{i-i}^{\top} \tilde{\mu}_{\alpha_n}  (C_n B_n^{\top}) u_{(t-1):0}  \left( \sigma_{i-1}^{1/4} \tilde{\phi}_{i-1} \right) } \\
    & = \norm{ \sum_{n = 1}^{d_A}  \tilde{\phi}_{i-i}^{\top} \tilde{\mu}_{\alpha_n}  (C_n B_n^{\top}) u_{(t-2):0}  \tilde{\phi}_{i-1}  } \\
    & \leq  d_A   \exp(-(i-1)/4\log(T))  \norm{ C_n B_n^{\top}} \norm{ \phi_{i-1} }_1  \\
    & \leq c' d_A \sqrt{T}  \exp(-(i-1)/4\log(T)) \norm{C} \norm{B}.
\end{align*}
Therefore, 
\begin{align*}
  \norm{\sum_{i = k+1}^{T} \mtrue_{i} u_{(t-1):0}  v_{i}} &  \leq c' d_A T^{3/2}  \exp(-i/4\log(T)) \norm{C} \norm{B}.
\end{align*}
Therefore as long as 
\begin{equation*}
  k  \geq  4 \log(T) \log \left( \frac{T^{3/2} c' d_A \norm{C} \norm{B} }{\epsilon}\right),
\end{equation*}
then 
\begin{equation*}
    \norm{\sum_{i = k+1}^{T} \mtrue_{i} u_{(t-1):0}  v_{i}} \leq \frac{\epsilon}{T}.
\end{equation*}
To prove $(4)$ we note that the statement is obvious for $i \leq 2$. For $i \geq 3$ the proof from Lemma D.1 of \cite{hazan2017learning} directly applies due to Lemma~\ref{lemma:nu_alpha_phi_small}. Next, Lemma~\ref{lemma:nu_l1_norm_filter} proves $(5)$.  Finally we prove $(6)$. Next, Lemma~\ref{lemma:nu_l1_norm_filter} proves $(5)$.  Finally we prove $(6)$. Since we have $h(\alpha) = (\alpha-1)^2$ and $\ell = 2$,
\begin{align}
\label{eqn:alpha_terms_vanilla_two}
    \max_{\alpha(A)} \left \{ \abs{ h(\alpha) \alpha^{L- 3} (1 - \alpha^{T - L + 1}) (1 -\alpha)^{-1}} \right \}  & =  \max_{\alpha(A)} \left \{ (1 - \alpha)\alpha^{L- 3} (1 - \alpha^{T - L + 1})  \right \}.
\end{align}
To bound Eq.~\ref{eqn:alpha_terms_vanilla_two}, consider the case where $\alpha$ is bounded away from $1$. Suppose $\alpha = 1 - \delta$, then
\begin{align*}
    (1 - \delta)^{L-3} \leq \frac{1}{T^p} \iff \log\left( \frac{1}{1 - \delta} \right)  \geq \frac{p \log(T)}{L - 3}.
\end{align*}
Observe that for $\delta \in [0, 1]$, $\log(1/(1-\delta)) \geq \delta/2$. Therefore, if
\begin{align*}
    \delta \geq \frac{2p \log(T)}{L - 3},
\end{align*}
we are guaranteed that $\alpha^{L-3} \leq 1/T^p$. Next consider when $\alpha$ is very close to $1$. To ensure that Eq.~\ref{eqn:alpha_terms_tensor_version} is bounded by $1/T^p$ we only require
\begin{equation*}
    \alpha \geq 1 - \frac{1}{T^p}.
\end{equation*}
Plugging in $p = 1/4$, we conclude that Eq.~\ref{eqn:alpha_terms_tensor_version} is bounded by $T^{-1/4}$ if 
\begin{equation*}
\alpha_n \in \left[ 0, 1 - \frac{ \log(T)}{2(L-3)} \right] \cup \left[  1 - \frac{1}{T^{1/4}}, 1  \right] \textrm{ for all } n \in [d_A].
\end{equation*}
\end{proof}

\subsection{Properties of the Hankel Matrix for Two Autoregressive Terms}
In Algorithm~\ref{alg:sf_two} we define
\begin{equation*}
\tilde{\mu}_{\alpha} \defeq (\alpha - 1)^2 \begin{bmatrix}  1 & \alpha &  \dots & \alpha^{T} \end{bmatrix}^{\top} \in \reals^{T}
\end{equation*}
and
\begin{equation*}
N_{T} = \int_{\alpha \in [0,1]} \tilde{\mu}_{\alpha} \tilde{\mu}_{\alpha}^{\top} d \alpha.
\end{equation*}
In what follows we present and prove several lemmas needed for the proof of Theorem~\ref{thm:lengthgeneralization_two_auto}.
\begin{lemma}[Properties of $N_T$]
\label{lemma:nu_alpha_phi_small}
    For any $\alpha \in [0,1]$ and $1 \leq i \leq T$,
    \begin{equation*}
       \max_{\alpha \in [0,1]} \abs{\phi_i^{\top} \tilde{\mu}_{\alpha}} \leq 6^{1/4} \sigma_i^{1/4}.
    \end{equation*}
\end{lemma}
\begin{proof}
    We have
    \begin{align*}
        \int_{\alpha \in [0,1]} \left( \phi_i^{\top} \tilde{\mu}_{\alpha} \right)^2 d \alpha 
        & = \phi_i^{\top} \left(  \int_{\alpha \in [0,1]}  \tilde{\mu}_{\alpha} \tilde{\mu}_{\alpha}^{\top}  d \alpha\right) \phi_i \\
        & = \phi_i^{\top} N_T \phi_i = \sigma_i.
    \end{align*}
    Next we observe that for $f_w(\alpha) \defeq \left( w^{\top} \tilde{\mu}_{\alpha} \right)^2 $, where $w$ is any unit-norm vector, we have that $f_w$ is $6$-Lipschitz on $[0,1]$. Indeed,
    \begin{align*}
        f_w'(\alpha) & = \frac{d}{d \alpha}  (\alpha - 1)^4 \left(  \sum_{i = 1}^T w_i \alpha^{i-1} \right)^2 \\
        & =  2 (\alpha - 1)^4 \left(  \sum_{i = 1}^T w_i \alpha^{i-1} \right) \left( \sum_{i = 2}^T (i-1) w_i \alpha^{i-2} \right) + 4 \left(  \sum_{i = 1}^T w_i \alpha^{i-1} \right)^2(\alpha - 1)^3 \\
        & \leq 2  (\alpha - 1)^4 \left( \frac{1 - \alpha^T}{1 - \alpha} \right) \left( \sum_{i = 1}^{T-1} i  \alpha^{i-1} \right) + 4 \left(  \frac{1 - \alpha^T}{1 - \alpha}\right)^2(\alpha - 1)^3 \\
        & =  2  (\alpha - 1)^4 \left( \frac{1 - \alpha^T}{1 - \alpha} \right) \left( \frac{1 - T \alpha^{T-1} + (T-1) \alpha^T}{(1 - \alpha)^2} \right) + 4 \left(  \frac{1 - \alpha^T}{1 - \alpha}\right)^2(\alpha - 1)^3 \\
        & =  2   \left(1 - \alpha^T \right) \left( 1 - T \alpha^{T-1} + (T-1) \alpha^T \right) + 4 \left(  1 - \alpha^T \right)^2(\alpha - 1) \\
        & \leq 2 + 4 = 6.
    \end{align*}
Consider any non-negative L-Lipschitz function $f$ that reaches some maximum value $g_{\textrm{max}}$ over $[0,1]$. The function $f$ which satisfies $L$-Lipschitzness, attains $g_{\textrm{max}}(f) $ and also has minimum possible area $A(f) \defeq \int_{\alpha \in [0,1]} f(\alpha) d \alpha$ is 
\begin{align*}
    f^*(\alpha) & = \begin{cases}
        L \alpha, & \textrm{ for } \alpha \in [0,\alpha^*] \\
        \max \left \{ g_{\textrm{max}} - L(\alpha - \alpha^*), 0 \right \} , & \textrm{ for } \alpha \in [\alpha^*, 1]
    \end{cases} \\
    & = \begin{cases}
        L \alpha, & \textrm{ for } \alpha \in [0,\alpha^*] \\
         g_{\textrm{max}} - L(\alpha - \alpha^*) , & \textrm{ for } \alpha \in [\alpha^*, \alpha^* + \frac{g_{\textrm{max}}}{L}] \\
         0, & \textrm{ for } \alpha \in [\alpha^* + \frac{g_{\textrm{max}}}{L}, 1]
    \end{cases}.
\end{align*}
Indeed, any oscillation away from this piecewise linear function would either increase the total area or violate the Lipschitz constraint. For this to be a valid construction we must have $L \alpha^* = g_{\textrm{max}}$ and therefore the minimum corresponding area is
\begin{equation*}
    A(f^*) = \int_{\alpha \in [0,1]} f^*(\alpha) d \alpha = \frac{1}{2} (\alpha^*)(L \alpha^*) + \frac{1}{2} (g_{\textrm{max}}/L) g_{\textrm{max}} = \frac{g_{\textrm{max}}^2}{L}.
\end{equation*}
And therefore for any function $f$ we have $g_{\textrm{max}}(f)  \leq \sqrt{L A(f)}$. Using this for $f_{\phi_i}(\alpha)$ we have
\begin{equation*}
    \max_{\alpha \in [0,1]} f_{\phi_i}(\alpha) = \max_{\alpha \in [0,1]} (\phi_i^{\top} \tilde{\mu}_{\alpha})^2 \leq \sqrt{6  \int_{\alpha \in [0,1]} \left( \phi_i^{\top} \tilde{\mu}_{\alpha} \right)^2 d \alpha } = \sqrt{6 \sigma_i}.
\end{equation*}
We conclude by noting
\begin{equation*}
    \max_{\alpha \in [0,1]} \abs{\phi_i^{\top} \tilde{\mu}_{\alpha}} = \sqrt{  \max_{\alpha \in [0,1]} (\phi_i^{\top} \tilde{\mu}_{\alpha})^2} \leq 6^{1/4} \sigma_i^{1/4}.
\end{equation*}
\end{proof}

\begin{lemma}[Adapted from Lemma $E.2$ from \cite{hazan2017learning}]
\label{lemma:exp_decay_nu_eigs}
    Let $\sigma_j$ be the $j$-th top singular value of $N_T$. Then for all $T \geq 10$ we have
    \begin{equation*}
        \sigma_j \leq \min \left( \frac{3}{2}, K \cdot c^{-j /\log(T)} \right),
    \end{equation*}
    where $c = e^{\pi^2/4} \approx 11.79$ and $K \leq 10^6$ is an absolute constant. 
\end{lemma}
\begin{proof}
    The proof provided in \cite{hazan2017learning} applies directly to $N_T$ with only one necessary modification to bound the trace. Observe that
    we have
\begin{align*}
    (N_{T})_{ij} & = \int_{\alpha \in [0,1]} (\alpha-1)^4 \alpha^{i+j-2} d \alpha \\
    & = \int_{\alpha \in [0,1]} \alpha^{i+j} - 2 \alpha^{i+j-1} + \alpha^{i+j-2} d \alpha \\
    & = \frac{24}{(i + j - 1) (i + j) (i + j + 1) (i + j + 2) (i + j + 3)}.
\end{align*}
Therefore,
    \begin{equation*}
        \sigma_j \leq \mathrm{tr}(N_T) = \sum_{i = 1}^T \frac{24}{(2i - 1) (2i) (2i + 1) (2i + 2) (2i+3)} \leq \sum_{i = 1}^T \frac{24}{(2i)^5 } = \frac{3}{4} \sum_{i = 1}^T \frac{1}{i^5} < \frac{3}{2}.
    \end{equation*}
    The remainder of the proof is an exact copy of the proof of Lemma $E.2$ with $3/4$ replaced by $3/2$.
\end{proof}

\begin{lemma}[Controlling the $\ell_1$ norm of the filters]
\label{lemma:nu_l1_norm_filter}
Let $(\sigma_j, \phi_j)$ be the $j$-th largest eigenvalue-eigenvector pair of $N_T$. Then for $T \geq 4$,
\begin{equation*}
    \norm{\phi_j}_1 \leq O \left( \frac{\log T}{\sigma_j^{1/4}}\right) .
\end{equation*}
\end{lemma}
\begin{proof}
This proof is a copy from the proof of Lemma E.5 in \cite{hazan2017learning} with only one noted modification. We note that $E$ as defined in their proof is entrywise bounded (for $T \geq 4$) by $24/T^5 \leq 2/T^3$ (which is the stated bound they use for their matrix of interest). We also must show the base case is true for $T_0 = 4$ instead of $T_0 = 2$. We have
\begin{align*}
    \norm{N_4^{1/4}}_{2 \to 1} = \sup_{x: \norm{x}_2 \leq 1} \norm{N_4^{1/4} x}_1 \leq \sum_{i,j = 1}^4 \abs{\left(N_4^{1/4}\right)_{ij}} < 2.
\end{align*}

We note that a tighter result is actually true for $N_T$ in that $\norm{\phi_j}_1 \leq O \left( \frac{\log T}{\sigma_j^{1/8}}\right) $. However, we omit this statement and proof because we don't leverage it for a tighter result overall.
    
\end{proof}

In Algorithm~\ref{alg:sf_two} we define
\begin{equation*}
\tilde{\mu}_{\alpha} \defeq (\alpha - 1)^2 \begin{bmatrix}  1 & \alpha &  \dots & \alpha^{T} \end{bmatrix}^{\top} \in \reals^{T}
\end{equation*}
and
\begin{equation*}
N_{T} = \int_{\alpha \in [0,1]} \tilde{\mu}_{\alpha} \tilde{\mu}_{\alpha}^{\top} d \alpha.
\end{equation*}
We have
\begin{align*}
    (N_{T})_{ij} & = \int_{\alpha \in [0,1]} (\alpha-1)^4 \alpha^{i+j-2} d \alpha \\
    & = \int_{\alpha \in [0,1]} \alpha^{i+j} - 2 \alpha^{i+j-1} + \alpha^{i+j-2} d \alpha \\
    & = \frac{24}{(i + j - 1) (i + j) (i + j + 1) (i + j + 2) (i + j + 3)}.
\end{align*}

\section{Proof for Length Generalization via Tensorized Spectral Filtering}
\label{appendix:tensor_proof}

The proof of Theorem~\ref{thm:lengthgeneralization_tensors} ultimately comes from Theorem~\ref{thm:general_length_gen} and its proof in Appendix~\ref{appendix:general}. Theorem~\ref{thm:general_length_gen} abstracts the necessary assumptions needed to obtain a length generalization guarantee. In Lemma~\ref{lemma:length_general_tensor} we prove that Algorithm~\ref{alg:tensor_sf} satisfies these assumptions. 
\begin{proof}[Proof of Theorem~\ref{thm:lengthgeneralization_two_auto}]
    By Lemma~\ref{lemma:length_general_tensor} and the assumptions made in the statement of Theorem~\ref{thm:lengthgeneralization_tensors}, we may apply Theorem~\ref{thm:general_length_gen} to Algorithm~\ref{alg:tensor_sf} to get that \begin{equation*}
    \sum_{t = 1}^T \ell_t(M^t, L) - \min_{M^* \in \K_{\norm{C} \norm{B}}} \sum_{t=1}^T \ell_t(M^* , T) \leq \left( 12 k^{3/2} \norm{C}^2  \norm{B}^2 \log^2(T) + 8 \norm{C}^2 \norm{B}^2 \right) \sqrt{T}.
\end{equation*}
\end{proof}

\begin{lemma}[Length Generalization for Tensorized Spectral Filtering]
\label{lemma:length_general_tensor}
Recall that in Algorithm~\ref{alg:tensor_sf} we let $T' = \left(\lceil \sqrt{T-2} \rceil \right)^2 + 2$ and $\phi_1, \dots, \phi_{\sqrt{T'-2}}$ be the orthonormal eigenvectors of $H_{\sqrt{T'-2}}$ with eigenvalues $\sigma_1, \dots, \sigma_{\sqrt{T'-2}}$. 
    Algorithm~\ref{alg:tensor_sf} is equivalent to Algorithm~\ref{alg:general_sf} with the following:
    \begin{enumerate}[label=(\alph*)]
    \item $p_t(y_{t-1:1}) = 2 y_{t-1} - y_{t-2}$
    \item $v_1 = e_1$, $v_2 = e_2$ and for $i, j \geq 1$,
    \begin{equation*}
        v_{i + k(j-1) + 2} = (0, 0, \sigma_{i}^{1/4} \phi_{i} \otimes \sigma_{j}^{1/4} \phi_j).
    \end{equation*}
\end{enumerate}
Define $\mtrue$ as follows:
\begin{equation*}
    \mtrue_1 \defeq CB,
\end{equation*}
\begin{equation*}
    \mtrue_2 \defeq C(A-2I) B,
\end{equation*}
and for $i,j \geq 1$,  
\begin{equation*}
    \mtrue_{i + k(j-1) + 2}  \defeq \sum_{n = 1}^{d_A} \frac{1 - \alpha_n}{1 - \alpha_n^{\sqrt{T}}} \left( \sigma_{i}^{-1/4} \phi_{i}^{\top} \mu_{\alpha_n} \right)  \left( \sigma_{j}^{-1/4} \phi_{j}^{\top} \mu_{\alpha_n^{\sqrt{T}}} \right)   (C_n B_n^{\top}).
\end{equation*}
Then the following properties hold 
\begin{enumerate}
\item For $h(A) = (A-I)^2$ and $\ell_1 = 2$
\begin{equation*}
     y_t - p_t(y_{t-1:1})  =  \sum_{i = 1}^{\ell_1} \mtrue_i u_{t-i} + \sum_{i = 1}^{t-\ell_1} C A^i h(A) B u_{t - \ell_1 - i}.
\end{equation*}
\item $y_t - p_t(y_{t-1:1}) = \sum_{i = 1}^T \mtrue_i u_{t-1:1} v_i$.
\item For $k = \Omega \left( \log(T) \cdot \log \left(d_A T \norm{C} \norm{B}/\epsilon  \right) \right) $,
\begin{equation*}
    \norm{\sum_{i = k+1}^T \mtrue_i u_{t-1:1} v_i} \leq \epsilon/T.
\end{equation*}
\item For any $i \in [T]$
\begin{equation*}
    \norm{\mtrue_i} \leq \norm{C} \norm{B}.
\end{equation*}
\item  For any $i \in [T]$, $\norm{v_i}_1 \leq \log^3(T)$ and $\{v_i \}_{i \in [T]}$ are orthonormal.
\item Finally if the spectrum of $A$ lies in the interval \begin{equation*}
    \left[ 0, 1 - \frac{ \log(T)}{2(L-2)} \right] \cup \left[  1 - \frac{1}{2 T^{1/4}}, 1  \right],
\end{equation*}
then
    \begin{equation*}
    \max_{\alpha(A)} \left \{ \abs{ h(\alpha) \alpha^{L- \ell_1 - 1} (1 - \alpha^{T - L + 1}) (1 -\alpha)^{-1}} \right \} \leq \frac{1}{T^{1/4}}.
\end{equation*}  
\end{enumerate}

\end{lemma}

\begin{proof}[Proof of Lemma~\ref{lemma:length_general_tensor}]
    Suppose $y_t$ evolves as an LDS. By definition, there exist matrices $(A, B, C)$ such that
\begin{equation}
\label{eqn:y_t_dyn}
    y_{t} = \sum_{i = 1}^t C A^{i-1} B u_{t-i}.
\end{equation}
Assume w.l.o.g. that $A$ is diagonal (otherwise replace $C$ with $CP$ and $B$ with $P^{-1}B$ where $P$ diagonalizes $A$). Observe that $(1)$ is true since unraveling Eq.~\ref{eqn:y_t_dyn} gives 
\begin{align*}
    y_t - 2 y_{t-1} + y_{t-2} & = CB u_{t-1} + C(A-2I)Bu_{t-2} + \sum_{i = 0}^{t-3} C  A^i(A^2 - 2A + I) B u_{t-3-i}.
\end{align*}
Let 
\begin{equation*}
\tilde{\mu}_\alpha^T \defeq (1 - \alpha)^2[1, \alpha, \alpha^2, \dots, \alpha^{T-1}],
\end{equation*}
and let
\begin{equation*}
    \mu_{\alpha}^T \defeq (1 - \alpha) [1, \alpha, \alpha^2, \dots, \alpha^{T-1}].
\end{equation*}
Observe that when $T$ is a perfect square,
\begin{equation*}
    \tilde{\mu}_\alpha^T = \frac{1-\alpha}{1-\alpha^{\sqrt{T}}} \mu_{\alpha}^{\sqrt{T}} \otimes \mu_{\alpha^{\sqrt{T}}}^{\sqrt{T}}.
\end{equation*}
Let $\alpha_1, \dots, \alpha_{d_A}$ denote the eigenvalues of $A$. Recall we set $T' = \left( \lceil \sqrt{T-2} \rceil \right)^2 + 2$ so that $\sqrt{T'-2}$ is an integer. We observe the following equality:
\begin{align*}
    \sum_{i = 0}^{t-3} C  A^i(A^2 - 2A + I) B u_{t-3-i}
    & =  \sum_{i = 0}^{t-3} C  \sum_{n = 1}^{d_A} \alpha_n^i (\alpha_n -1)^2 e_n e_n^{\top} B u_{t-3-i} \\
    & = \sum_{n = 1}^{d_A}   \left( C e_n e_n^{\top} B\right)    \sum_{i = 0}^{t-3} \alpha_n^i (\alpha_n -1)^2 u_{t-3-i} \\
    & = \sum_{n = 1}^{d_A}   \left( C_n B_n^{\top} \right)     u_{(t-3):0} \tilde{\mu}_{\alpha_n}^{T'-2}  \\
    & = \sum_{n = 1}^{d_A}  \frac{1 - \alpha_n}{1 - \alpha_n^{\sqrt{T'-2}}} \left( C_n B_n^{\top} \right)     u_{(t-3):0} \left( \mu_{\alpha_n}^{\sqrt{T'-2}} \otimes \mu_{\alpha_n^{\sqrt{T'-2}}}^{\sqrt{T'-2}} \right).
\end{align*}
Recall that $\left \{ \phi_i \right \}$ are the orthonormal eigenvectors of the Hankel matrix $H_{\sqrt{T'-2}}$. Observe that $\left \{ \phi_i \otimes \phi_j \right \}$ is an orthonormal set, inheriting this property from $\left \{ \phi_i \right \}$. Therefore, 
\begin{align*}
     u_{(t-3):0} \left( \mu_{\alpha_n} \otimes \mu_{\alpha_n^{\sqrt{T'-2}}} \right) & =  u_{(t-3):0} \left( \sum_{i,j = 1}^{\sqrt{T'-2}} (\phi_i \otimes \phi_j) (\phi_i \otimes \phi_j)^{\top} \right) \left( \mu_{\alpha_n} \otimes \mu_{\alpha_n^{\sqrt{T'-2}}} \right) \\
     & = \sum_{i,j = 1}^{\sqrt{T'-2}} \phi_i^{\top}\mu_{\alpha_n} \phi_j^{\top} \mu_{\alpha_n^{\sqrt{T'-2}}} u_{(t-3):0}  (\phi_i \otimes \phi_j) 
\end{align*}Therefore, letting 
\begin{equation*}
    M_{ij} \defeq \sum_{n = 1}^{d_A} \frac{1 - \alpha_n}{1 - \alpha_n^{\sqrt{T'-2}}}  \phi_i^{\top} \mu_{\alpha_n} \phi_j^{\top} \mu_{\alpha_n^{\sqrt{T'-2}}} \left( C_n B_n^{\top} \right),
\end{equation*}
\begin{align*}
     \sum_{n = 1}^{d_A}   \frac{1 - \alpha_n}{1 - \alpha_n^{\sqrt{T'-2}}} \left( C_n B_n^{\top} \right)     u_{(t-3):0} \left( \mu_{\alpha_n} \otimes \mu_{\alpha_n^{\sqrt{T'-2}}} \right) 
 & = \sum_{i,j = 1}^{\sqrt{T'-2}} M_{ij}  u_{(t-3):0}  (\phi_i \otimes \phi_j) \\
 & = \sum_{i,j = 1}^{\sqrt{T'-2}} \mtrue_{i + k(j-1) + 2} u_{(t-3):0}  ( \sigma_i^{-1/4} \phi_i \otimes \sigma_j^{-1/4} \phi_j) \\
 & = \sum_{\ell = 3}^{T} \mtrue_{\ell} u_{(t-1):1}  v_{\ell}.
\end{align*}
This concludes the proof of $(2)$. To prove $(3)$ we use Lemma 13.4 in \cite{hazan2022introduction}: there is some universal constant $c'$ such that,
\begin{equation*}
   \max_{\alpha \in [0,1]} \abs{\phi_{i}^{\top} \mu_{\alpha}} \leq c' (\sqrt{T'-2})^2 \exp(-i/\log(\sqrt{T'-2})) = c' T\exp(-2i/\log(T)).
\end{equation*}
Suppose 
\begin{equation*}
    k \geq \frac{1}{2} \log(T) \cdot \log \left(c' d_A T^{7/2} \norm{C} \norm{B}/\epsilon  \right) 
\end{equation*}
Let $\ell = i + k(j-1) + 2$ and suppose either $i \geq k$ or $j \geq k$ (and therefore $\ell \geq k + 2$). Then, using that for $T > 3$, $\max_{\alpha \in [0,1]} (1-\alpha)/(1 - \alpha^{\sqrt{T'-2}}) \leq 1$, 
\begin{align*}
    \norm{ \mtrue_{\ell} u_{(t-1):0} v_{\ell} }  & = \norm{ \sum_{n = 1}^{d_A} \frac{1 - \alpha_n}{1 - \alpha_n^{\sqrt{T}}}\left( \sigma_{i}^{-1/4} \phi_{i}^{\top} \mu_{\alpha_n} \right)  \left( \sigma_{j}^{-1/4} \phi_{j}^{\top} \mu_{\alpha_n^{\sqrt{T}}} \right)   (C_n B_n^{\top}) u_{(t-3):0}  \left( \sigma_{i}^{1/4} \phi_{i} \otimes \sigma_j^{1/4} \phi_j \right) } \\
    & \leq d_A \left( c' (T'-2) \exp(-2k/\log(T'-2)) \right) \norm{C} \norm{B} \norm{u_{(t-3):0}}_{\infty} \norm{\phi_i \otimes \phi_j}_1 \\
    & \leq  c' d_A T^{3/2} \exp(-2k/\log(T)) \norm{C} \norm{B} \\
    & \leq \epsilon/T^2.
\end{align*}
Therefore,
\begin{align*}
     \norm{ \sum_{\ell = k + 2}^{T'-2} \mtrue_{\ell} u_{(t-1):1} v_{\ell} } \leq \epsilon/T.
\end{align*}
The proof of $(4)$ comes from Lemma D.1 of \cite{hazan2017learning}. The proof extends directly to 
\begin{equation*}
    \mtrue_{ij} \defeq \sum_{n = 1}^{d_A} \frac{1 - \alpha}{1 - \alpha^{\sqrt{T'-2}}} \phi_i^{\top} \mu_{\alpha_n} \phi_j^{\top} \mu_{\alpha_n^L} (C_n B_n^{\top}),
\end{equation*}
since it proceeds by bounding $\max_{i \in [k]} \max_{\alpha \in [0,1]} \phi_i^{\top} \mu_{\alpha}$ by Lemma E.4. To prove $(5)$ we use Lemma~\ref{lemma:phi_is_sparse} from \cite{hazan2017learning}.

Therefore for $v_3, \dots, v_T$ we have for some universal constant $c > 0$,
\begin{align*}
    \norm{ v_{i + k(j-1) + 2}}_1 & = \norm{ \sigma_i^{1/4} \phi_i \otimes \sigma_j^{1/4} \phi_j }_1  \\
    & \leq \sigma_i^{1/4} \sigma_j^{1/4} \norm{\phi_i}_1 \norm{\phi_j}_1 \\
    & \leq c \log^2(T) \\
    & \leq \log^3(T),
\end{align*}
where the last inequality holds for $T$ large enough. \\
Finally we prove $(6)$. Since we have $h(\alpha) = (\alpha-1)^2$ and $\ell = 2$,
\begin{align}
\label{eqn:alpha_terms_tensor_version}
    \max_{\alpha(A)} \left \{ \abs{ h(\alpha) \alpha^{L- 3} (1 - \alpha^{T - L + 1}) (1 -\alpha)^{-1}} \right \}  & =  \max_{\alpha(A)} \left \{ (1 - \alpha)\alpha^{L- 3} (1 - \alpha^{T - L + 1})  \right \}.
\end{align}
To bound Eq.~\ref{eqn:alpha_terms_tensor_version}, consider the case where $\alpha$ is bounded away from $1$. Suppose $\alpha = 1 - \delta$, then
\begin{align*}
    (1 - \delta)^{L-3} \leq \frac{1}{T^p} \iff \log\left( \frac{1}{1 - \delta} \right)  \geq \frac{p \log(T)}{L - 3}.
\end{align*}
Observe that for $\delta \in [0, 1]$, $\log(1/(1-\delta)) \geq \delta/2$. Therefore, if
\begin{align*}
    \delta \geq \frac{2p \log(T)}{L - 3},
\end{align*}
we are guaranteed that $\alpha^{L-3} \leq 1/T^p$. Next consider when $\alpha$ is very close to $1$. To ensure that Eq.~\ref{eqn:alpha_terms_tensor_version} is bounded by $1/T^p$ we only require
\begin{equation*}
    \alpha \geq 1 - \frac{1}{T^p}.
\end{equation*}
Plugging in $p = 1/4$, we conclude that Eq.~\ref{eqn:alpha_terms_tensor_version} is bounded by $T^{-1/4}$ if 
\begin{equation*}
\alpha_n \in \left[ 0, 1 - \frac{ \log(T)}{2(L-3)} \right] \cup \left[  1 - \frac{1}{T^{1/4}}, 1  \right] \textrm{ for all } n \in [d_A].
\end{equation*}
   
\end{proof}

\end{document}